\documentclass{article}
\PassOptionsToPackage{numbers, compress}{natbib}

\usepackage[preprint]{neurips_2021}

\usepackage{relsize}
\usepackage[utf8]{inputenc} 
\usepackage[T1]{fontenc}    

\usepackage{url}            
\usepackage{booktabs}       
\usepackage{amsfonts}       
\usepackage{nicefrac}       
\usepackage{microtype}      
\usepackage{multicol}
\usepackage{wrapfig}

\PassOptionsToPackage{table}{xcolor}
\usepackage[matit]{echodefs}
\usepackage{thmtools, thm-restate}
\usepackage{todonotes}
\usepackage{dirtytalk}
\usepackage{graphicx}
\usepackage{cleveref}
\usepackage{tcolorbox}
\usepackage[bottom]{footmisc}
\usepackage{enumitem}

\definecolor{Gray}{gray}{0.9}
\newcommand{\etal}{\textit{et al. }}

\usepackage{hyperref}       

\newcommand{\card}{\mathrm{card }}

\def\bL{{\mathbb L}}
\def\bS{{\mathbb S}}
\def\bE{{\mathbb E}}
\def\bH{{\mathbb H}}

\def\acos{{\mathrm{acos}}}
\def\acosh{{\mathrm{acosh}}}
\DeclarePairedDelimiter\ceil{\lceil}{\rceil}
\DeclarePairedDelimiter\floor{\lfloor}{\rfloor}

\title{Geometry of Similarity Comparisons}

\author{%
	Puoya Tabaghi\\
	Coordinated Science Lab\\
  ECE Department, UIUC\\
    \texttt{tabaghi2@illinois.edu} \\
  \And
	Jianhao Peng\\
	Coordinated Science Lab\\
  ECE Department, UIUC\\
    \texttt{jianhao2@illinois.edu} \\
  \And
	Olgica Milenkovic \\
  Coordinated Science Lab\\
  ECE Department, UIUC\\
  \texttt{milenkov@illinois.edu} \\
  \AND
  Ivan Dokmani\'c\\
  Department of Mathematics and Computer Science\\
  University of Basel\\
  \texttt{ivan.dokmanic@unibas.ch}
}

\begin{document}

\maketitle
\begin{abstract}
Many data analysis problems can be cast as distance geometry problems in \emph{space forms} -- Euclidean, spherical, or hyperbolic spaces. Often, absolute distance measurements are often unreliable or simply unavailable and only proxies to absolute distances in the form of similarities are available. Hence we ask the following: Given only \emph{comparisons} of similarities amongst a set of entities, what can be said about the geometry of the underlying space form? To study this question, we introduce the notions of the \textit{ordinal capacity} of a target space form and \emph{ordinal spread} of the similarity measurements. The latter is an indicator of complex patterns in the measurements, while the former quantifies the capacity of a space form to accommodate a set of measurements with a specific ordinal spread profile. We prove that the ordinal capacity of a space form is related to its dimension and the sign of its curvature. This leads to a lower bound on the Euclidean and spherical embedding dimension of what we term similarity graphs. More importantly, we show that the statistical behavior of the ordinal spread random variables defined on a similarity graph can be used to identify its underlying space form. We support our theoretical claims with experiments on weighted trees, single-cell RNA expression data and spherical cartographic measurements.
\end{abstract}

\section{Introduction}\label{sec:intro}
Distances reveal the geometry of their underlying space. They are at the core of many machine learning algorithms. In particular, finding a geometrical representation for point sets based on pairwise distances is the subject of distance geometry problems (DGPs). Euclidean DGPs have a rich history of applications in robotics~\cite{porta2005branch,tabaghi2019kinetic}, wireless sensor networks~\cite{so2007theory}, molecular conformation analysis~\cite{crippen1988distance} and dimensionality reduction~\cite{liberti2014euclidean}. One is typically concerned with finding a geometric representation for a set of measured Euclidean distances~\cite{dokmanic2015euclidean}. Beyond Euclidean DGPs, recent works have focused of hyperbolic geometry methods in data analysis, most notably when dealing with hierarchical data. Social and FoodWeb networks~\cite{verbeek2014metric,li2017inhomogeneous}, gene ontologies~\cite{ashburner2000gene}, and Hearst graphs of hypernyms~\cite{le2019inferring} are interesting examples of hierarchical datasets. Spherical embeddings represent sets of points on a (hyper)sphere~\cite{wilson2010spherical}, and have found applications in astronomy \cite{green1985spherical}, distance problems on Earth \cite{bai2015constrained}, and texture mapping \cite{elad2005texture}. Euclidean, spherical and hyperbolic geometries are categorical examples of constant curvature spaces, or space forms, which are characterized by their curvature and dimension. The above examples represent instances of metric embeddings in space forms, as opposed to what is termed \emph{nonmetric embeddings}. In the latter setting, one is provided with nonmetric information about data points, such as quantized distances or ordinal measurements such as comparisons or rankings.

We argue that nonmetric information such as \emph{distance comparisons} carries valuable information about the space the data points originated from. To formally state our claims, assume that we are given a set of points $x_1, \ldots, x_N$ in an unknown metric space $S$. In nonmetric embedding problems \cite{kruskal1964nonmetric,agarwal2007generalized}, we work with dissimilarity (similarity) measurements of the form
\[
\forall m,n \in [N] \bydef \set{1, \ldots, N}: y_{m,n} = \phi \big( d(x_m, x_n) \big),
\]
where $d(x_m,x_n)$ is the distance between $x_m$ and $x_n$ in $S$, and $\phi(\cdot)$ is an unknown monotonically increasing (or decreasing) function. Since $\phi$ is unknown, we can \emph{only} interpret the measurements as distance comparisons or ordinal measurements, i.e., if the entities indexed by $n_1,n_2$ are more similar than those indexed by $n_3,n_4,$ then
\[
y_{n_1, n_2} \leq y_{n_3,n_4} \xLeftrightarrow{ \ \ \mbox{increasing} \ \ \mathlarger{\phi} \ \ }  d(x_{n_1}, x_{n_2}) \leq   d(x_{n_3}, x_{n_4}).
\]
We hence ask: \emph{What do distance/similarity comparisons as those described above reveal about the space $S$?} Our work shows, for the first time, that one can use ordinal measurements to deduce the sign of the curvature and a lower-bound for the dimension of the underlying space form (in Euclidean and spherical spaces). The main results of our analysis are as follows.
\begin{enumerate}[leftmargin=*]
\item We introduce the notion of \emph{ordinal spread} of the sorted distance list, which is of fundamental importance in the study of the geometry of distance comparisons. The spread of ordinal measurements describes a pattern in which entities appear in the sorted list of distances, i.e., the ordinal spread gives the ranking of the first appearance of a data point in the list. 
This notion is related to another important geometric entity termed the \emph{ordinal capacity}.
\item We define the notion of \emph{ordinal capacity} of a space form to characterize the space's ability to host extreme patterns of ordinal spreads (computed from similarity measurements). We show that the ordinal capacity of a space form is related to its dimension and curvature sign. The ordinal capacity of Euclidean and spherical spaces are equal and grow exponentially with their dimensions, while the ordinal capacity of a hyperbolic space is infinite for any possible dimension of the space. 
\item We derive a deterministic lower bound for Euclidean and spherical embedding dimensions using ordinal spreads and the (finite) ordinal capacity. We also associate an \emph{ordinal spread random variable} with $(1)$ a set of random points in a space form, and $(2)$ a set of random vertex subsets from a similarity graph -- a complete graph with edge weights corresponding to similarity scores of their defining nodes. The distributions of these random variables serves as a practical tool to identify the underlying space form given a similarity graph.  
\item We illustrate the utility of our theoretical analysis by using them to correctly uncover the hyperbolicity of weighted trees. Moreover, we use them to detect Euclidean and spherical geometries for ordinal measurements derived from local and global cartographic data. Finally, we use the ordinal spread variables to determine the degree of heterogeneity of cell populations based on noisy scRNAseq data and how data imputation influences the geometry of the cell space trajectories. 
\end{enumerate} 
Due to space limitations, all proofs, algorithms, and extended discussions are delegated to the Supplement.
\subsection{Related Works}
In many applications we seek a representation for a group of entities based on their distances, but the exact magnitudes of the distances may be unavailable.
What often \emph{is} available (and prevalent) in applied sciences are nonmetric -- dissimilarity or similarity -- measurements: In neural coding~\cite{giusti2015clique}, developmental biology~\cite{klimovskaia2020poincare}, learning from perceptual data~\cite{demiralp2014learning}, and cognitive psychology~\cite{navarro2004common}. Unfortunately, the datasets used in most of these studies are small (often involving less than $100$ entities) and have limited utility for learning tasks that require sufficiently large sample complexity.

Nonmetric embedding problems date back to the works of Shepard~\cite{shepard1962analysis,shepard1962analysis2} and Kruskal \cite{kruskal1964nonmetric}. Inspired by the Shepard-Kruskal scaling problem, Agarwal~\etal\cite{agarwal2007generalized} introduce generalized nonmetric multidimensional scaling, a semidefinite relaxation used to embed dissimilarity (or similarity) ratings of a set of entities in Euclidean space. Stochastic triplet embeddings~\cite{van2012stochastic} and crowd kernel learning~\cite{tamuz2011adaptively} are used to embed triadic comparisons using probabilistic information. Tabaghi~\etal\cite{tabaghi2020hyperbolic} propose a semidefinite relaxation for metric and nonmetric embedding problems in hyperbolic space. In all these scenarios, the embedding space has to properly represent the measured data. For example, in developmental biology and cancer genomics, single-cell RNA sequencing (scRNAseq) is used to differentiate cell types and cycles. The classification results have important implications for lineage identification and monitoring cell trajectories and dynamic cellular processes~\cite{tanay2017scaling}. Klimovskaia~\etal\cite{klimovskaia2020poincare} use hyperbolic rather than Euclidean spaces for low-distortion embedding of complex cell trajectories (hierarchical structures). 

Learning from distance comparisons is an active area of research.
Among the relevant research topics are ranking objects from pairwise comparisons \cite{wauthier2013efficient,jamieson2011active}, theoretical analysis of necessary number of distance comparisons to uniquely determine the embedding \cite{jamieson2011low}, nearest neighbor search \cite{haghiri2017comparison}, random forests \cite{haghiri2018comparison}, and classification based on triplet comparisons \cite{cui2020classification}. Understanding the underlying geometry of ordinal measurements is important in designing relevant algorithms. 

Related to nonmetric embedding problems are the various techniques that study topological properties of point clouds independently of the choice of metric and of the geometric properties such as curvature~\cite{carlsson2009topology}. An important problem in this domain is to detect intrinsic structure in neural firing patterns, invariant under nonlinear monotone transformations of measurements. Giusti~\etal~\cite{giusti2015clique} propose a method based on clique topology of the graph of correlations between pairs of neurons. The clique topology of a weighted graph describes the behavior of cycles in its order complex~\cite{giusti2015clique} as a function of edge densities; these entities are also known as \emph{Betti curves}. The statistical behavior of Betti curves is used to distinguish random and geometric structures of moderate sizes in Euclidean space. The more recent work of Zhou~\etal~\cite{zhou2018hyperbolic} generalizes this statistical approach to hyperbolic spaces. These two works are the most closely related contributions to our proposed problem area. Nevertheless, the technical approaches used in there and in our work are fundamentally different. First, we provide a theoretical foundation for the study of geometric properties of space forms using similarity comparisons and derive the first known rigorous results related to their dimensions and curvatures. Second, we propose a computationally efficient method for inferring the sign of the curvature. The proposed statistical method can operate on large datasets as it uses subsampling techniques. Furthermore, we introduce new application areas in outlier identification, heterogeneity detection and imputation analysis for single-cell data measurements. To the best of our knowledge, we report the first study regarding the effect of different imputation degrees on the geometry of similarity measurements in these datasets. 
\section{The Ordinal Spread} \label{sec:ordinal_spread}
\emph{Preliminaries.}
A space form is a complete, simply connected Riemannian manifold of dimension $d \geq 2$ and constant sectional curvature. Up to an isomorphism, space forms are equivalent to spherical $(\bS^d)$, Euclidean $(\bE^d)$, or hyperbolic spaces $(\bH^d)$~\cite{lee2007riemannian}. Distance geometry problems (DGPs) are concerned with finding an embedding for a set of pairwise measurements in a space form. DGP problems can be categorized as metric~\cite{tabaghi2020hyperbolic}, nonmetric~\cite{agarwal2007generalized}, or unlabeled \cite{skiena1990reconstructing,jones2004introduction,huang2018reconstructing}, depending on the data modality and application domain. A nonmetric DGP aims to find $x_1, \ldots, x_N$ in a space form $S$, given a set of ordinal distance measurements $\mathcal{O} \subseteq [N]^4$ such that 
\begin{equation}\label{eq:ordinal_measurements}
\forall (n_1, n_2, n_3, n_4) \in \mathcal{O}: d(x_{n_1},x_{n_2}) \leq d(x_{n_3},x_{n_4}).
\end{equation}
Although there exist theoretical results on the uniqueness of Euclidean embeddings~\cite{kleindessner2014uniqueness} (up to an ordinal invariant transformation, i.e., an isotony), most often the underlying geometry of ordinal measurements is not known a priori~\cite{klimovskaia2020poincare,cao2013similarity,mcfee2011learning}. 

We consider the problem of identifying the underlying space form from a given set of pairwise distance comparisons. For sufficiently many comparisons, this problem is equivalent to inferring geometrical information through the \emph{sorted distance list} associated with ordinal measurements~\eqref{eq:ordinal_measurements}. A deterministic or a randomized binary sort algorithm needs at least $\Theta({N \choose 2} \log {N \choose 2} )$ pairwise comparisons to uniquely find the sorted distance list, if such a list exists~\cite{cormen2009introduction}. Hence, we can define the \emph{sorted index list} 
$(i_r,j_r)_{r \in {N \choose 2}}$ according to
\begin{equation}\label{eq:sorted_distance}
d(x_{i_1},x_{j_1}) \geq \cdots \geq d(x_{i_{{N \choose 2}}},x_{j_{{N \choose 2}}}),
\end{equation}
where $i_{r} < j_{r}$ for all $ r \in \big[ {N \choose 2} \big]$ and all pairs of indices are distinct. Any geometry-related inference problem must be \emph{invariant} with respect to arbitrary permutations of the point indices. In particular, the pattern of the newly added indices in the sorted index list~\ref{eq:sorted_distance} is invariant to the permutations of point indices and has important geometrical implications. We formalize this notion in~\Cref{def:kth_ordinal_spread}.
\begin{tcolorbox}[colback=blue!4!white,colframe=blue!4!white]
\begin{definition}\label{def:kth_ordinal_spread}
The $n$-th ordinal spread of $N$ points with a sorted index list is defined as
\begin{equation*}
\forall n \in [N]: \alpha_{n} = \min \set{m \in \mathbb{N}: \card  \ \bigcup_{r=1}^{m} \set{i_{r}, j_{r}} \geq n }.
\end{equation*}
\end{definition}
\end{tcolorbox}
Alternatively, the ordinal spread $\alpha_n$ is the rank of the first appearance of the $n$-th point in the sorted index list, i.e.,
$\card  \ \bigcup_{r=1}^{\alpha_{n}-1} \set{i_{r}, j_{r}} < n, \ \card  \ \bigcup_{r=1}^{\alpha_{n}} \set{i_{r}, j_{r}} \geq n.$
As an example, for $d(x_{1},x_{2}) \geq d(x_{1},x_{3}) \geq \cdots$, we have $\alpha_{3}=2$. From~\Cref{def:kth_ordinal_spread}, we observe that one can compute the ordinal spread $\alpha_n$ without knowing the point set positions, the distance and $\phi(\cdot)$ function or even the type of underlying space. 
Nevertheless, in~\Cref{sec:stylized_experiments} and later on, we use $\alpha_n ( \set{x_n}_{n=1}^{N} )$ to denote the $n$-th ordinal spread computed for the points $\set{x_n}_{n=1}^{N}$ in a metric space.

In general, the ordinal spreads $\set{ \alpha_n}_{n \in [N]}$ depend on the configuration of the underlying point set, up to a similarity preserving map~\cite{kleindessner2014uniqueness}. In \Cref{prop:kth_osn}, we make the first step in studying ordinal spread variables by computing their range of possible values.
\begin{tcolorbox}[colback=red!4!white,colframe=red!4!white]
\begin{restatable}[]{proposition}{kth_osn} \label{prop:kth_osn}
For a set of $N \geq 4$ points with a given sorted index list, we have 
\begin{itemize}
\begin{minipage}{0.5\linewidth}
    \item $\alpha_{1}  = \alpha_{2} = 1$, $\alpha_{3}  = 2$; 
\end{minipage}
\begin{minipage}{0.5
\linewidth}
    \item $4 \leq n \leq N: \ceil*{\frac{n}{2}} \leq \alpha_{n} \leq {n-1 \choose 2}+1.$
\end{minipage}
\end{itemize}
\end{restatable}
\end{tcolorbox}
Clearly, the $N$-th ordinal spread variable, $\alpha_N$, is the largest ordinal spread value which makes it a good choice for inferring geometry-related properties. In comparison, $\alpha_1$, $\alpha_2$, and $\alpha_3$ are fixed and independent on the space and hence noninformative (see the Supplement for more details). We next provide two illustrative examples that show how the ordinal spread $\alpha_N$ may be used to reveal the hyperbolic, Euclidean, and spherical geometry of the measurements. These results motivate the study of ordinal capacity.
\vspace{-5pt}
\subsection{Hyperbolicity of Trees} \label{sec:stylized_experiments}
\begin{figure}[b]
	\includegraphics[width=1 \linewidth]{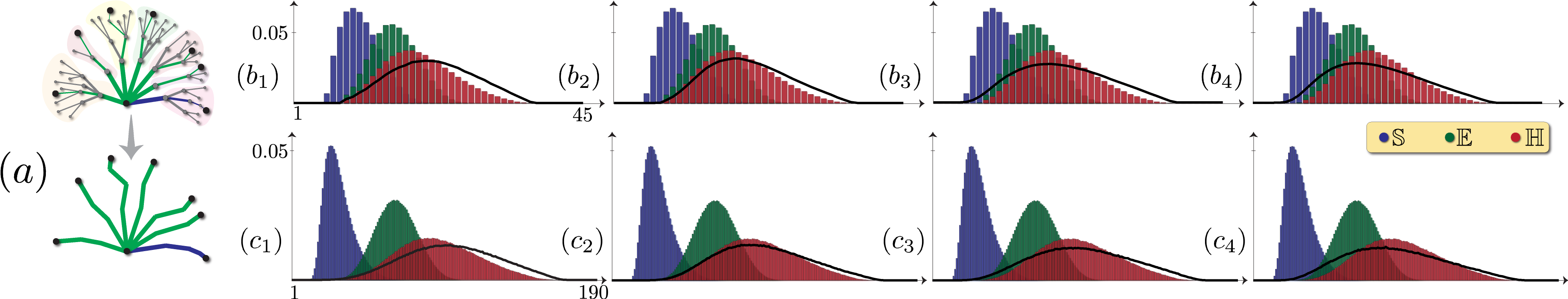}
	\vspace{-5pt}
  	\caption{$(a)$ Random selection of a sub-tree of size $N$. PMFs of $\alpha_{10}$ (top row) and $\alpha_{20}$ (bottom row) for random points in $\bH^2$ (red), $\bE^2$ (green), and $\bS^2$ (blue). The black plots are empirical PMFs of $\alpha_{N}$ derived from $(b_1, c_1)$ the noise-less tree $T$, $(b_2, c_2)$ the additive noise contaminated tree, $(b_3, c_3)$ the tree with permutation noise, and  $(b_4, c_4)$ a tree with both previous forms of noise.}
  	\label{fig:tree_example}
\end{figure}
Hyperbolic spaces are space forms that offer small distortion when embedding trees \cite{sarkar2011low,ganea2018hyperbolic}. Here, we describe how to verify this hyperbolicity by using \emph{ordinal spread random variables}. We generate a random tree $T$ with the vertex set $V = [10^4]$. The maximum node degree is $3$ and edge weights are i.i.d. realizations of a $\mathrm{unif}(0,1)$-distributed random variable. Let $d_{m, n}$ be the distance between nodes $m$ and $n$ in $V$, defined as the sum of the weights on the unique path connecting the vertices. Then, we randomly subsample $10^6$ different node subsets, or sub-cliques, of size $N \in \set{10, 20}$ ($N \ll |V| = 10^4$) from $T$ as shown in~\Cref{fig:tree_example} $(a)$. For each randomly selected sub-clique, we compute its $N$-th ordinal spread, $\alpha_N$. Due to the inherently random nature of the clique selection process, $\alpha_N$ is a random variable which we term the ordinal spread random variable for the tree $T$. We can then compute the empirical distribution of the random variable, as illustrated in~ \Cref{fig:tree_example}$(b_1, c_1)$. This motivates the following definition:
\begin{tcolorbox}[colback=blue!4!white,colframe=blue!4!white]
\begin{definition}\label{def:ordinal_spread_variable}
Let $S$ be a metric space, and $P$ be a probability distribution on $S$. With a slight abuse of notation, we define the ordinal spread random variable $\alpha_N$ as
\[
\forall N \in \mathbb{N}: \alpha_N = \alpha_N(X), \ X \sim P^{\otimes N}.
\]
\end{definition}
\end{tcolorbox}
An ordinal spread random variable is defined with respect to the distribution $P$. Let us assume an \emph{oracle} picks a set of distributions for embedded points in each space form, e.g., (projected) normal for hyperbolic and Euclidean spaces, and uniform distribution in the spherical space. The distribution of the corresponding ordinal spread random variable  $\alpha_{N}$ is invariant with respect to scaling. More precisely, it is invariant to strongly isotonic point transformations (more information in the Supplement). 
As the results in~\Cref{fig:tree_example}$(b_1, c_1)$ indicate, the empirical distribution of $\alpha_N$ derived from a weighted tree $T$ best matches (in the sense of total variation distance between the probability measures) with that of a random hyperbolic point set. For further verification, we repeated the same experiment for a  random tree $T$ with $(1)$ additive measurements noise, e.g., $\tilde{d}_{m,n} = d_{m,n}+\eta$ where $\eta$ is a sample of a zero-mean Gaussian noise (with $20$ decibel signal-to-noise ratio), $(2)$ random permutation noise for the sorted index lists, e.g., $\tilde{i} = \pi(i)$ and $\tilde{j} = \pi(j)$ where $\pi$ is a permutation with average displacement of $|V| = 10^4$, and $(3)$ both additive and permutation noise; see~\Cref{fig:tree_example}$(b_2, c_2), (b_3, c_3)$, and $(b_4, c_4)$. The results clearly show that the distribution of the ordinal spread variable $\alpha_N$ is robust to noise and that it closely matches with that of a random hyperbolic point cloud. 
An important implication of this example is that ordinal spread variables can be used to determine the curvature sign of the underlying space. A more rigorous justification is provided in the subsequent exposition in~\Cref{sec:ordinal_capacity,sec:stylized_experiments2}, where we formally connect the support of ordinal spread variables to a specific property of their underlying space forms, i.e., their \emph{ordinal capacity}. In the Supplement, we show how to use ordinal capacity to compute a deterministic lower bound for the Euclidean embedding dimension of this tree.
\subsection{Euclidean and Spherical Geometries of Cartographic Data}
\begin{figure}[b]
	\includegraphics[width=1 \linewidth]{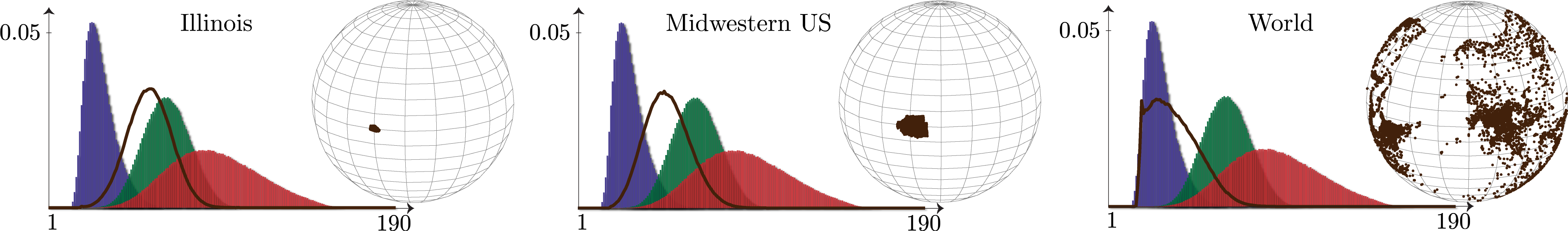}
	\vspace{-5pt}
  	\caption{The empirical PMFs of $\alpha_{20}$ derived from subsampling the dissimilarity (distance) graph associated with points in the state of Illinois, across the Midwestern USA, and the world. Colored plots are PMFs of for random points in $\bH^2$ (red), $\bE^2$ (green), and $\bS^2$ (blue).} 
  	\label{fig:city_distances}
\end{figure}
We describe next an experiment pertaining to the ordinal spread (random) variables of a similarity graph for geospatial data. The main idea is to use the distribution of these variables to show that the intrinsic geometry of small regions on the globe, which are ``flat,'' is close to Euclidean, whereas that of large regions, which are spread across the globe, are close to spherical. 

We use three datasets: $(1)$ $1,627$ counties in the state of Illinois, $(2)$ $11,954$ counties in Midwestern states, and $(3)$ $10^4$ (subsampled) cities and towns across the world; refer to the Supplement for details on data sources. We construct the dissimilarity graph by computing the pairwise distances between the points using the Haversine formula, which determines the great-circle distance between two points on the globe given their longitudes and latitudes \cite{van2012heavenly}. For each data set, we compute the empirical PMF of $\alpha_{20}$ from $10^6$ randomly selected cliques of size $20$ each; the results are shown in~\Cref{fig:city_distances}. Comparing the PMFs for $\alpha_{20}$ and for random hyperbolic, Euclidean, and spherical points, we clearly observe the shift from an (approximately) Euclidean to a spherical geometry as the area spanned by the sampled points increases. We emphasize that these results are derived from distance comparisons only, since we discard the metric information in the distances.
\vspace{-5pt}
\section{The Ordinal Capacity}\label{sec:ordinal_capacity}
\vspace{-5pt}
In the numerical experiment of \Cref{sec:stylized_experiments}, we discovered a distinguishing statistical behavior for the ordinal spread of randomly generated points in each possible space form. We show in what follows that this distinguishing pattern is related to the \emph{capacity} of each space form to accommodate ordinal spread random variables with their underlying distributions. We define \textit{ordinally dense sets} and show how they can help determine the support (the range of possible values) of the ordinal spread random variables in a space form\footnote{We adopt Mirsky's notation $\set{m,n}_{\neq}$ for a set with two distinct elements $m$ and $n$ \cite{mirsky1971transversal}.}.
\begin{tcolorbox}[colback=blue!4!white,colframe=blue!4!white]
\begin{definition}\label{def:ordinal_capacity}
Let $\{{x_1, \ldots, x_N\}}$ be a set of distinct points in a metric space $S$. If
\[
\exists n_0 \in [N]: \sup_{n \in [N]\setminus \set{n_0}} d(x_{n}, x_{n_0}) \leq \inf_{\set{m,n}_{\neq} \subseteq [N]\setminus \set{n_0}} d(x_{m}, x_{n}),
\]
then we say that $\set{x_n}_{n=1}^{N}$ is an ordinally dense set in $S$, or in short $\set{x_n}_{n=1}^{N} \sqsubseteq S$.
\end{definition}
\end{tcolorbox}
In a nutshell,~\Cref{def:ordinal_capacity} identifies point configurations that have a maximum possible ordinal spread. Intuitively, a set of $N$ points is ordinally dense in $S$ if and only if it has a subset of $N-1$ points whose pairwise distances are {\bf all} larger than (or equal to) their distances to the $N$-th point, i.e.,
\begin{equation}\label{eq:ordinally_dense_ordinal_spread}
\set{x_n}_{n=1}^{N} \sqsubseteq S \Longleftrightarrow \alpha_N \left( \set{x_n}_{n=1}^{N} \right) = {N-1 \choose 2}+1.
\end{equation}
\begin{figure}[b]
	\center
	\includegraphics[width=0.6 \linewidth]{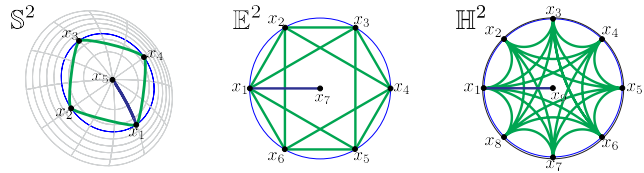}
	\vspace{-5pt}
  	\caption{Ordinally dense point sets in $2$-d space forms. As all distances in the (Euclidean) hexagon are greater than or equal to their distances to the center, the point set achieves the capacity $K(\mathbb{E}^2) = 7$. }
  	\label{fig:non_metric_illustration}
  	\vspace{-5pt}
\end{figure}
The existence of an ordinally dense set of size $N$ depends on the geometry of the underlying metric space, and is closely tied to what we term the \emph{ordinal capacity} of the space (see~\Cref{fig:non_metric_illustration}). 
\begin{tcolorbox}[colback=blue!4!white,colframe=blue!4!white]
\begin{definition}\label{def:ordinal_capacity_number}
The ordinal capacity of a metric space $S$ is defined as
\[
K (S) = \sup \set{ \card \set{x_n} :  \set{x_n} \sqsubseteq S}.
\]
\end{definition}
\end{tcolorbox}
The ordinal capacity is an indicator of the capability of a metric space to realize an extremal pattern of point indices in the sorted index list~\eqref{eq:ordinally_dense_ordinal_spread}. In the next theorem, we show that the ordinal capacity of a space form is intimately related to a spherical cap packing problem~\cite{rankin1955closest}, which is concerned with the maximum number of non-overlapping spherical caps (or domes with a certain polar angle) in a hypersphere.
\begin{tcolorbox}[colback=red!4!white,colframe=red!4!white]
\begin{restatable}[]{theorem}{ocn}
\label{thm:ordinal_capacity_number}
Let $N_d$ be the spherical $\frac{\pi}{6}$-cap packing number of $\mathbb{S}^d$. The ordinal capacity of a space form $S$ is given by
\begin{equation*}
K (S) = 
\begin{cases}
\begin{aligned}
&+\infty, &&\mbox{ if } S \cong \bH^d   \\
& N_d+1, &&\mbox{ if } S \cong \bE^d, S \cong \bS^d.
\end{aligned}
\end{cases}
\end{equation*}
\end{restatable}
\end{tcolorbox}
\Cref{thm:ordinal_capacity_number} shows that the ordinal capacity of space forms depends on their curvature sign and dimension. The ordinal capacity of a hyperbolic space is infinite, regardless of its dimension. This implies that for any $N \in \N$, there exists an ordinally dense hyperbolic point set $\set{x_n}_{n=1}^{N}$. In the Poincar\'e model, a centered regular $(N-1)$-gon with an additional point in the ``center'' is an ordinally dense set (see~\Cref{fig:non_metric_illustration}). In contrast, Euclidean and spherical spaces have \emph{equal and finite} ordinal capacities. This finding is intuitively clear because any tangent space of $\bS^d$ is a linear space of dimension $d$, and the spherical distance converges to the $\ell_2$ distance as the distance between the points diminishes. In the Supplement, we propose a refinement for the ordinal capacity of spherical spaces by imposing a minimum distance constraint for the point sets. We note that the current notion of ordinal capacity does not distinguish between hyperbolic spaces of different dimensions. Therefore, one may need to develop a more refined notion of ordinal capacity for hyperbolic spaces, e.g., based on extremal appearance patterns of \emph{multiple} indices in the distance lists.

Using the previous result, we can numerically compute an upper bound on of $N_d$, $\rho_d$, as a function of $d$, e.g. $\rho_1 = 2, \rho_2 = 6, \rho_3 = 15, \rho_4 = 31, \rho_5 = 59, \rho_6 = 106$~\cite{rankin1955closest}. Note that the packing number $N_d$ grows exponentially with the dimension $d$~\cite{wyner1967random}. Hence, we have the following assymptotic bound for the ordinal capacity of a $d$-dimensional Euclidean (or spherical) space:
\[
	- \log\big(\frac{\sqrt{3}}{2}\big) +o(d) \leq \frac{1}{d}\log K(\bE^d) \leq -\log\big(\frac{\sqrt{2}}{2}\big)  +o(d).
\]
\section{The Support of Ordinal Spread Random Variables} \label{sec:stylized_experiments2}
In \Cref{sec:ordinal_spread}, we showed numerical evidence that ordinal spread random variables in Euclidean, spherical, and hyperbolic geometries have different supports. We therefore ask: \emph{What is the maximum achievable ordinal spread, $\alpha_N$, for a point set of size $N > K(S)$?} The answer to this question determines the support of  ordinal spread random variables in Euclidean and spherical spaces, regardless of their underlying distribution $P$ (see~\Cref{def:ordinal_spread_variable}). Note that since the ordinal capacity of a hyperbolic space is infinite, there always exists a point set of size $N$ with maximum ordinal spread of ${N-1 \choose 2}+1$ (see~\Cref{prop:kth_osn}). For our subsequent analysis, we define the \emph{$N$-point ordinal spread of a space form $S$} to be the maximum attainable ordinal spread $\alpha_N$ for the points in $S$.
In~\Cref{thm:space_form_n_point_curvature}, we express this quantity in terms of the ordinal capacity of $S$.
\begin{tcolorbox}[colback=red!4!white,colframe=red!4!white]
\begin{restatable}[]{theorem}{sfnp_curvature}
\label{thm:space_form_n_point_curvature}
The $N$-point ordinal spread of a space form $S$ is given by
\[
A_N(S) \bydef \sup_{X \in S^N} \alpha_{N} \left( X \right) = E \big( T(N-1,K(S)-1) \big) +1,
\] 
where $E \big( T(N,K) \big)$ is the number of edges of $T(N,K)$, the $K$-partite Tur\'an graph~\cite{turan1941external} with $N$ vertices. 
\end{restatable}
\end{tcolorbox}
As a conclusion, the $N$-point ordinal spread of a space form, i.e., the support of its ordinal spread random variable $\alpha_N$, depends on its ordinal capacity and the number of points $N$. For a space $S$ with finite ordinal capacity, there exists a point set $X  \in S^N$ such that $\alpha_N(X) < {N-1 \choose 2}+1$. This holds if $N > K(S)$. With this result, we can revise the ordinal spread bound in~\Cref{prop:kth_osn}.
\begin{tcolorbox}[colback=red!4!white,colframe=red!4!white]
\begin{restatable}[]{proposition}{kth_osn} \label{prop:kth_osn_v2}
For a set of $N \geq 4$ points in a space form $S$, we have 
\begin{itemize}
\begin{minipage}{0.5\linewidth}
    \item $\alpha_{1}  = \alpha_{2} = 1$, $\alpha_{3}  = 2$, 
\end{minipage}
\begin{minipage}{0.5
\linewidth}
    \item $4 \leq n \leq N: \ceil*{\frac{n}{2}} \leq \alpha_{n} \leq A_n(S).$
\end{minipage}
\end{itemize}
\end{restatable}
\end{tcolorbox}
\Cref{thm:space_form_n_point_curvature} and~\Cref{prop:kth_osn_v2} explain in part the discriminatory ability of the support of ordinal spread random variables across different space forms. The $N$-point ordinal spread of a hyperbolic space $\bH^d$ is the maximum value possible, i.e., $A_{N}(\bH^d) = {N-1 \choose 2}+1$, regardless of its dimension. Even though the $N$-point ordinal spread of Euclidean and spherical spaces, $\mathbb{E}^d$ and $\mathbb{S}^d$, varies with their dimension, they are equal to each other. This is evident from our distribution-free analysis of the ordinal capacity of these spaces (see \Cref{thm:ordinal_capacity_number} and its subsequent discussion).
However, we can extend our distribution-free results to the following coarse lower bound for Euclidean (or spherical) embedding dimension of a similarity graph, 
\[
\min \set{d: \sup_{X \subseteq V: |X| = N}\alpha_N(X) \leq A_N(\mathbb{E}^d), \forall N \in [|V|]} \leq d
\]
where $V$ is the vertex set of the graph. We may relax an exhaustive search over all $2^N$ vertex subsets, to a search over a random subselection of vertices. In the Supplement we use such a relaxation to compute a lower bound for the embedding dimension of the tree discussed in \Cref{sec:stylized_experiments}.
\subsection{Visualizing Point Sets with Maximum Ordinal Spread}
\begin{figure}[b]
\center{
\vspace{-0.1in}
  \includegraphics[width=.77 \linewidth]{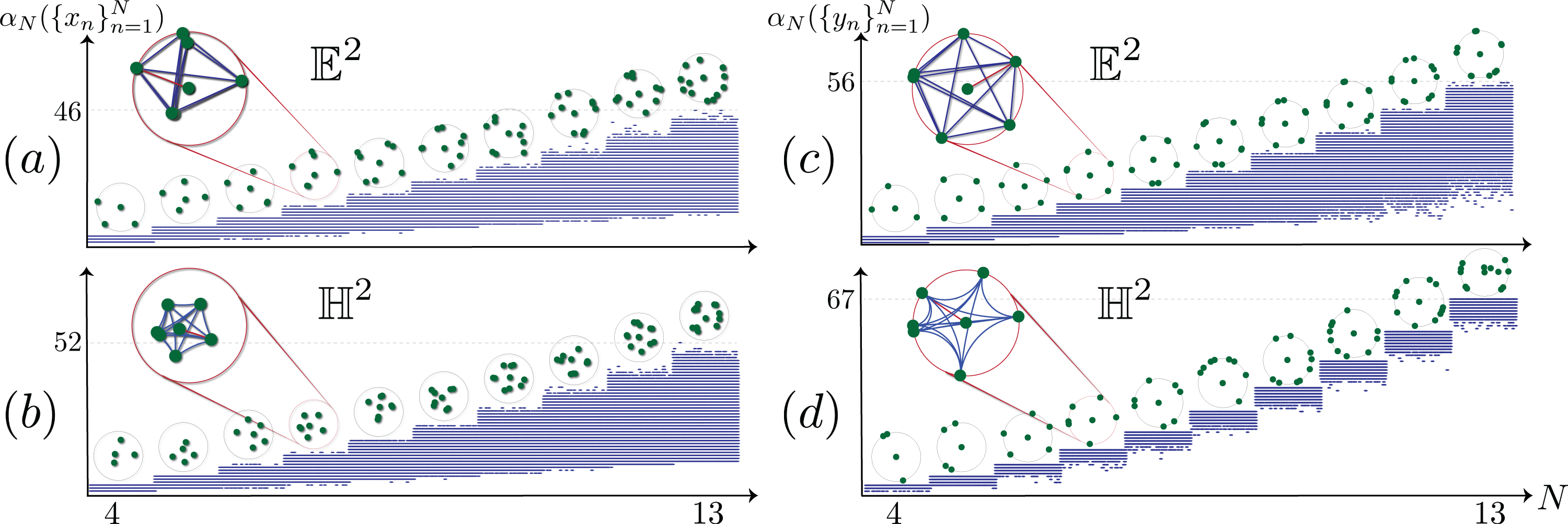}}
  \caption{Ordinal spread of $5\times 10^5$ i.i.d. point sets in $\bE^2$ and $\bH^2$. For fixed $N$, we mark the set with the maximum ordinal spread: $\set{x_n}_{n=1}^{N}$ in Figures $(a,b)$ and $\set{y_n}_{n=1}^{N}$ in Figures $(c,d)$.}
  \label{fig:oin_toy}
\end{figure}
Here, we aim to gain geometrical intuition about the point sets with maximum ordinal spread in different space forms. To this end, we generate independent and identically distributed point sets from a (projected) normal distribution in $2$-dimensional hyperbolic and Euclidean spaces. For each realization $\set{x_n}_{n=1}^{N}$, we compute the corresponding ordinal spread $\alpha_N$. The maximum ordinal spread of the generated point sets, $\widehat{A}_{N}$, gives an estimate for $A_N(\bE^2)$ and $A_N(\bH^2)$ (see ~\Cref{thm:space_form_n_point_curvature}). We repeat this experiment for varying size of the point sets $N \in \set{4,5, \ldots, 13}$.

For $5 \times 10^5$ realizations, we pick the point configurations with maximum ordinal spread; see~\Cref{fig:oin_toy} $(a,b)$. Recall that the point set with the theoretical maximum ordinal spread must have $N-1$ points sampled from a sphere centered at the $N$-th point. So, we repeat this experiment by fixing a point at $0$, and projecting the remaining points to their circumscribed circle, i.e., $\forall n \in [N-1]: y_n = \frac{r}{\norm{x_n}}x_n, \ \ \mbox{and}  \ \ y_N = 0$, where $r = \max_{n \in [N-1]}\norm{x_n} $. The randomly selected points $\set{y_n}_{n=1}^{N}$ produce a more accurate estimate for $A_N(\bH^2)$ and $A_N(\bE^2)$; see~\Cref{fig:oin_toy} $(c,d)$. For example, we have $\widehat{A}_{13}(\bE^2) = 56$, compared to the theoretical bound $A_{13}(\bE^2) \leq 58$. Also, the estimated $N$-point ordinal spread of a hyperbolic space perfectly matches with the theoretical bound $A_N(\bH^2) = {N-1 \choose 2}+1$. The latter result is due to the capacity of hyperbolic spaces to host infinitely many ordinally dense point sets. Hence,  the probability of randomly selecting an ordinally dense hyperbolic point set, of size $N$, is greater than its Euclidean counterpart.

Perhaps the most important observation is that the individual points in the extremal sets aggregate on nonoverlapping spherical caps of a circle, as seen in~\Cref{fig:oin_toy} $(c)$. The ordinal capacity of a space form equals the total number of such caps plus one (for the center point), i.e., $N_d +1$. For example, there are $5$ strictly non-overlapping spherical caps for $2$-dimensional Euclidean space, whereas this number is infinite for hyperbolic spaces. Finally, these results illustrate that the $N$-th ordinal spread of each space form, $A_N(S)$, is the total number of edges in Tur\'an graphs (see~\Cref{thm:space_form_n_point_curvature,thm:ordinal_capacity_number}). 

\vspace{-5pt}
\section{Numerical Experiments: Single-cell RNA Sequencing Data}\label{sec:Numerical_Results}
\vspace{-5pt}
Here, we focus on results pertaining to an important new data format omnipresent in computational molecular biology: \emph{single-cell RNA sequencing} (scRNAseq) data. By using recently developed single-cell isolation and barcoding techniques, and by trading individual cell coverage for the number of cells captured, scRNA-seq data for the first time enables studying the activity patterns of millions of individual cells. This is in stark contrast with traditional bulk sequencing techniques that only provide \emph{averaged snapshots} of cellular activity; scRNAseq measurements are also of special importance in cancer biology, as cancer cells are known to contain highly heterogeneous cell populations and the degree of heterogeneity carries significant information about disease progressions and the effectiveness of treatments~\cite{meacham2013tumour}. Important for our study is the fact that due to the large number of different cells sequenced, cell measurements are extremely sparse and \emph{imputed} in practice~\cite{van2017magic,li2018accurate,hicks2018missing}. 

Further, it has been pointed out~\cite{hicks2018missing, eraslan2019single} that scRNA-seq data is very noisy due to biological stochasticity as well as dropouts and systemic noise. Existing methods for denoising and imputation of raw scRNA-seq data often involve building connection graphs among cells~\cite{li2018accurate,van2017magic} using the distance between cells to diffuse the expression profiles among neighboring cells and smooth out possible outliers. Thus, \emph{relative expression differences (comparisons)}, rather than \emph{absolute expression values}, enable more accurate biological data mining via clustering, lineage detection, or inference of pseudotemporal orderings of cells~\cite{plass2018cell}. As an example,~\cite{klimovskaia2020poincare} constructs similarity probabilities from a relative forest accessibility (RFA) matrix~\cite{chebotarev1997matrix} and uses the obtained values to suggest that hyperbolic spaces are more suitable than Euclidean spaces for scRNAseq data embedding. 
\begin{figure}[t]
  \center
  \includegraphics[width=1 \linewidth]{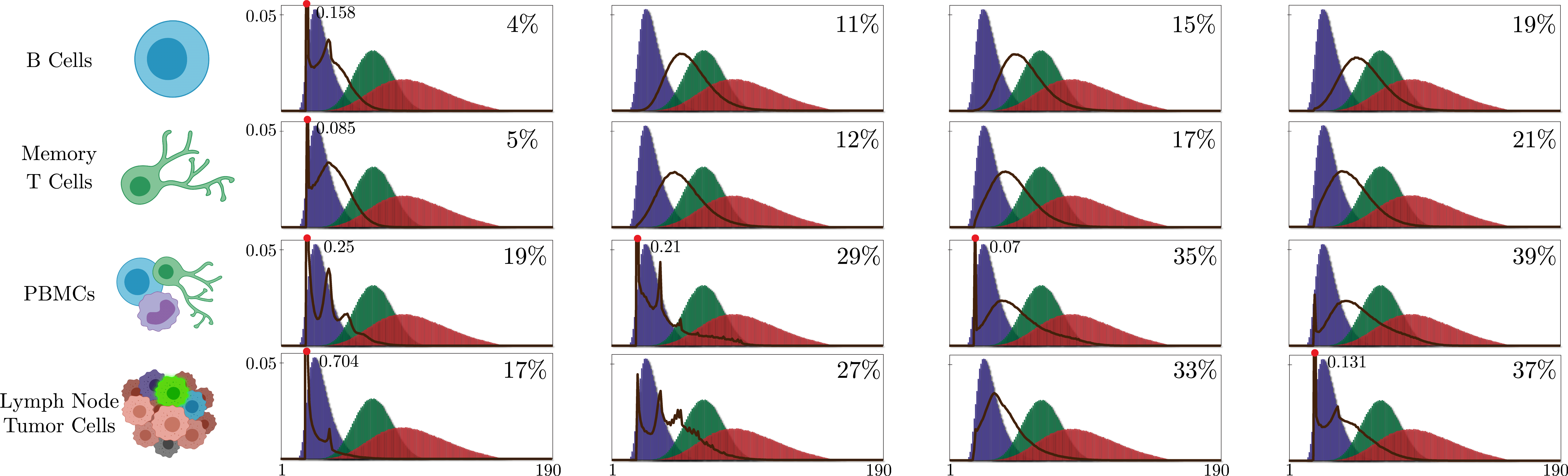}
    \vspace{-5pt}
    \caption{The empirical PMFs of $\alpha_{20}$ derived from subsampling the RFA similarity graph associated with scRNAseq data from homogeneous B cells ($\approx 10,000$ cells) and memory T cells ($\approx 10,000$), and heterogeneous PBMCs ($\approx 10,000$) and lymph node tumor cells ($\approx 3,000$). The left column shows the results for the raw data. From left to right, we increase the percentage of imputed data (densities are shown in the top-right corner).}
  \label{fig:RNAseq}
    \vspace{-14pt}
\end{figure}
We illustrate next that identifying the geometric properties of scRNAseq data using comparisons also provides unique information about the diversity of cellular populations~\cite{meacham2013tumour}, outliers and the properties of imputation methods. Furthermore, since scRNA captures temporal hierarchical information about cells, as well as the cyclic nature of cell cycles, we expect spherical space forms to be equally useful as hyperbolic space forms in the process of embedding. To this end, we compute the empirical distribution of ordinal spread random variables associated with scRNA lymphoma (cancer) cells and cell \emph{families} known as mononuclear cells (PBMCs), comprising T cells, B cells, and monocytes, which are often targeted in cancer immunotherapy. In this case, as illustrated by our numerical findings in~\Cref{fig:RNAseq}, these distributions contain peaks for small values that indicate that the data us sparse and contains outliers or highly heterogeneous cellular populations. Intuitively, probability peaks for small values of $\alpha_N$ arise when newly added indices in the ordered distance list appear in quick succession which can be attributed to one or multiple points at large distance from the remaining points (outliers); for more details see the Supplement. As imputation adds new data points by using averaged and smoothed information of observed measurements, it is expected to remove peaks in the aforementioned distributions, which is clearly the case for homogeneous cellular populations, but not for cancer cells and PBMCs. The reason why imputation does not remove peaks for the latter two categories can be attributed to the fact that the peaks arise due to the presence of many different cell types (e.g., recall that PBMCs contain B,T and monocytes and consequently, multiple peaks are observed in the ordinal spread distributions of raw data) which cannot and should not be smoothed out to form one class as this defeats the purpose of using single-cell measurements. Equally importantly, the results show that the Magic imputation software we used~\cite{van2017magic} imputes information into the noisy measurements without changing the geometry of the data, which is an important indicator of the quality of the procedure. 
\begin{figure}[b]
\vspace{-10pt}
  \begin{center}
    \includegraphics[width=0.70\textwidth]{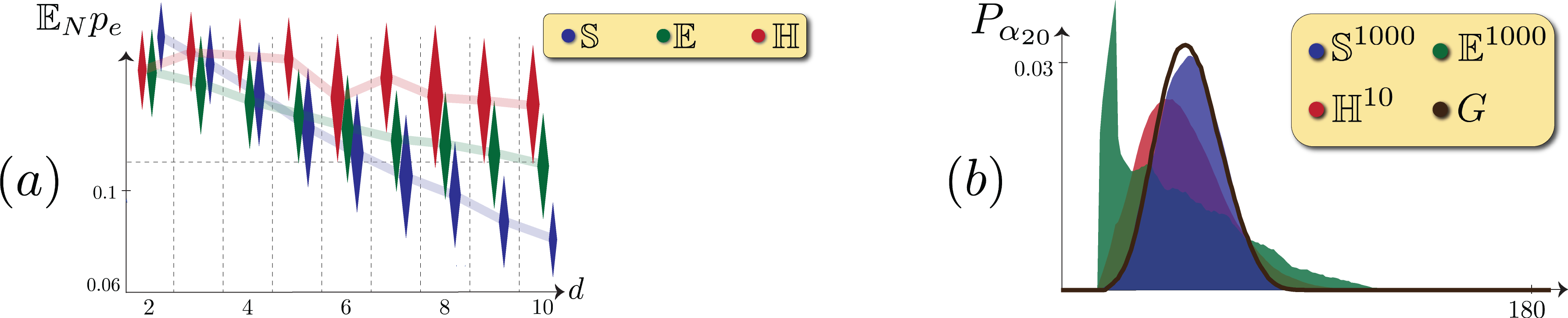}
  \end{center}
  \vspace{-10pt}
  \caption{$(a)$ PMFs of $\alpha_{20}$ from the RFA similarity graph ($G$) vs. random points in space forms of optimal dimensions. $(b)$ $\mathbb{E}_N p_e$ for embedded points in $d$-dimensional space forms.}
    \vspace{-12pt}
  \label{fig:nonmetric_embedding}
\end{figure}
Our next results pertain to the actual embedding quality of the measured similarities. We consider nonmetric embeddings~\cite{tabaghi2020hyperbolic,agarwal2007generalized} of RFA scores of scRNAseq data from adult planarians~\cite{plass2018cell}. The data set contains $N \approx 26,000$ cells with gene expression vectors of dimension $d \approx 21,000$. In~\Cref{fig:nonmetric_embedding} $(a)$, we report the empirical probability of incorrect comparison $\mathbb{E}_{N} p_e$ for embedding RFA similarities in different space forms of varying dimensions. The results thus confirm that a \emph{spherical} geometry is actually better suited for accurate nonmetric embeddings, which supports the frequently ignored understanding that cells are measured at various stages of the same cell cycle. For our analysis, we compute $\widehat{P}_{\alpha_{20}}$ from the similarity graph $G$. Due to the heavy-tailed nature of the original data distribution, we choose (oracle) log-normal distributions for the points in each space form. Then, we repeat the experiments for various dimensions and each space form/distribution parameters to find the closest ordinal spread variable to $\widehat{P}_{\alpha_{20}}$. From~\Cref{fig:nonmetric_embedding} $(b)$, we conclude that an ordinal spread variable from a high-dimensional ($\approx 1,000$) \emph{spherical space} best matches $\widehat{P}_{\alpha_{20}}$. For the details about the embedding methods, datasets, and further discussions refer to the Supplement.
\begin{ack}
Parts of the work were funded by the NSF grant number 2008125.
\end{ack}
\bibliography{bibfile}
\newpage
\section*{\huge \center SUPPLEMENTARY MATERIALS}

\tableofcontents
\newpage
\section{PROOFS OF THEOREMS AND PROPOSITIONS}

\paragraph{Notation}
For any two numbers $a,b \in \R$, we let $a \lor b$ and $a \land b$ be their maximum and minimum. We use small letters for vectors, $x \in \R^{m}$, and capital letters for matrices, $X = (x_{i,j})\in \R^{m \times n}$. We denote the $m$-th standard basis vector in $\R^M$ by $e_m$, $m \in [M]$ and let $[M]$ be short for the set $\set{1, \ldots, M}$.  For vectors $x,y \in \R^{d+1}$, their dot product is denoted by $\langle x, y \rangle $, and their Lorentzian inner product is $[x,y] = -x_0y_0 + \sum_{i=1}^{d}x_i y_i$. The $d$-dimensional 'Loid model of hyperbolic space is a Riemannian manifold $\mathbb{L}^d = \set{x \in \R^d: [x,x] = -1}$ with the distance function given by $d(x,y) = \acosh (-[x,y])$. Finally, $0$ and $1$ are all-zero and all-one vectors of appropriate dimensions. Let $C$ be a subset of a space form $(S,d)$, and $x \in S$; We define
\begin{align*}
d_{\min}(C) &= \inf \set{d(x,y): x,y \in C, x \neq y},  \\
d_{\max}(x,C) &= \sup \set{d(x,y): y \in C}.
\end{align*}
The cardinality of a discrete set $C$ is denoted by $\card \ C$. The graph-theoretic notations simplifies the main results of this paper. For a graph $G$, we denote its edge set as $E(G)$. Let $G_{p_1, \ldots, p_K}$ be a complete $K$-partite graph with part sizes $p_1, \ldots p_K$. The Tur\'an graph~\cite{turan1941external} $T(N,K)$ is a complete $K$-partite graph with $N$ vertices, and part sizes \footnote{From $\sum_{k=1}^{K}p_k = N$, we have $N_1 =  \floor*{ \frac{N}{K} }$, $K_1 = N-K N_1$. }
\begin{align*}
p_{k} = \begin{cases}
N_1+1, &\mbox{ for } \ 1\leq k \leq K_1\\
N_1, &\mbox{ for } \ K_1+1 \leq k \leq K.
\end{cases}
\end{align*}
Then, $\card \ E\big(T(N,K) \big) = {N \choose 2} - K_1 {N_1+1 \choose 2}  - \left( K-K_1 \right) {N_1 \choose 2}$. \footnote{This is simplified from $\card \ E \big( G_{p_1, \ldots, p_K} \big) = {N \choose 2}  - \sum_{k=1}^{K} {p_k \choose 2}$. For $K>N$, we assume the graph is complete and $E(T(N,K)) = {N \choose 2}$.}
\subsection{Proof of Proposition 1}  \label{sec:kth_osn}
From Definition 1, the values for $\alpha_1(X), \alpha_2(X)$ and $\alpha_3(X)$ are trivial. The lower bound for $\alpha_N( X )$ simply follows from the uniqueness of  pairwise distances. To put formally, we have
\[
\alpha_{N} \left( X\right) =   \min \set{m \in \mathbb{N}: \card  \ \bigcup_{r=1}^{m} \set{i_{r}, j_{r}} = N } \geq \ceil*{\frac{N}{2}}.
\]
For the upper bound, $\alpha_N( X )$ is maximum when all $N-1$ smallest pairwise distances are incident to a unique point; see Figure 2 $(b)$. The total length of the distance list is $ {N \choose 2}$. Therefore, we have
\[
\alpha_{N} \left( X \right) \leq {N \choose 2} -(N-1)+1 = {N-1 \choose 2} +1.
\]

{\bf Remark.} Let us devise an experiment to show how the $k$-th ordinal spread can distinguish space forms. We randomly generate i.i.d. points $\set{x_n}_{n=1}^{N}$ from absolutely continuous distributions with full support in hyperbolic, Euclidean and spherical spaces.\footnote{Uniform distribution for spherical, and projected normal for hyperbolic space, i.e., $y = [\sqrt{1+\norm{x}^2}, x^\T]^\T$ where $x \sim \mathcal{N}(0,\sigma^2 I)$.} In \Cref{fig:ordinal_test}, we plot the $k$-th ordinal spread $\alpha_k$ for each realization $\set{x_n}_{n=1}^{N}$. We find the empirical maximum of $\alpha_{N}$ to be the most sensitive indicator of the geometry of underlying space. While the emerging pattern of $\alpha_{N}$'s is dependent on the distribution of point sets, the behavior of empirical maximum of $N$-th ordinal spread is robust to the choice of point set distributions, as it converges to its supremum almost surely. Therefore, we introduce the \emph{$N$-point ordinal spread} for a metric space -- a tool to categorize space forms based on their ability to house extremal ordinal patterns, in the sense of  Theorem~2.
\begin{figure}[t]
  \includegraphics[width=1 \linewidth]{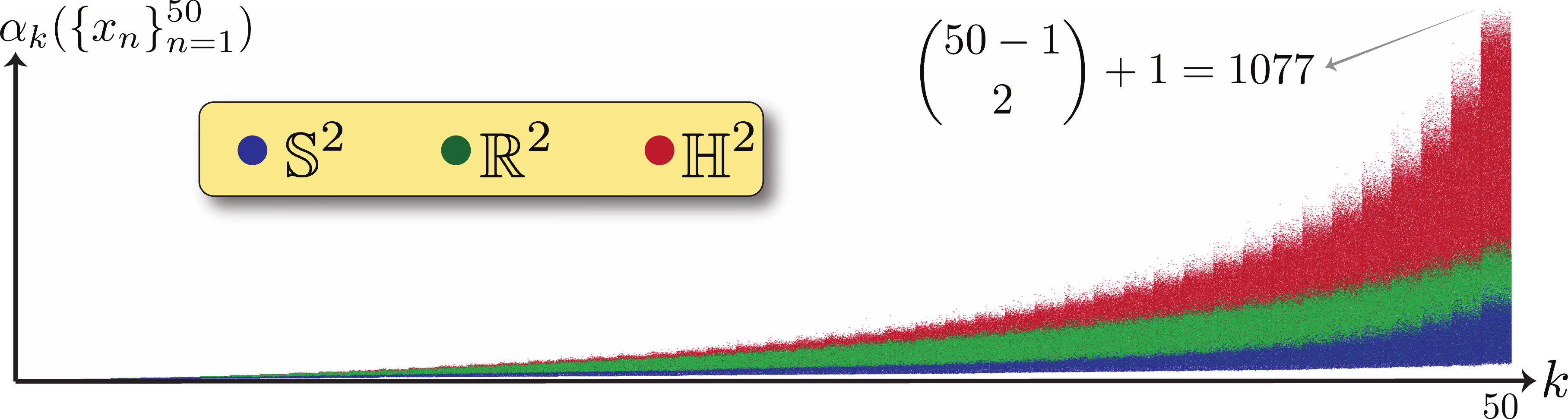}
  \caption{The $k$-th ordinal spread of $10^{5}$ randomly generated points $\set{x_n}_{n=1}^{50}$ in $2$-dimensional space forms. }
  \label{fig:ordinal_test}
\end{figure}
\newpage
\subsection{Proof of Theorem 1} \label{sec:ordinal_capacity_number}
Let us separately consider hyperbolic, Euclidean, and spherical spaces. 
\subsubsection{Hyperbolic space} 
Let $r \in \R^{+}$, and $x_1(r), \ldots, x_{N}(r) \in \bL^d$ be a set of parameterized points in 'Loid model of $d$-dimensional hyperbolic space, such that 
\[
\forall n \in [N]: x_n(r) = \begin{bmatrix} \sqrt{1+\norm{y_n(r)} ^2} \\ y_n(r)  \end{bmatrix}, 
\]
where $y_N(r) = 0$, and $ y_{i}(r)^\T y_{j}(r) = r^2 \cos{ 2\pi \frac{|i-j|}{N-1}}, \forall i,j \in [N-1]$. To see an example, see \Cref{fig:maximum_spread_hyperbolic}.
 \begin{figure}[b]
\center
  \includegraphics[width=.6 \linewidth]{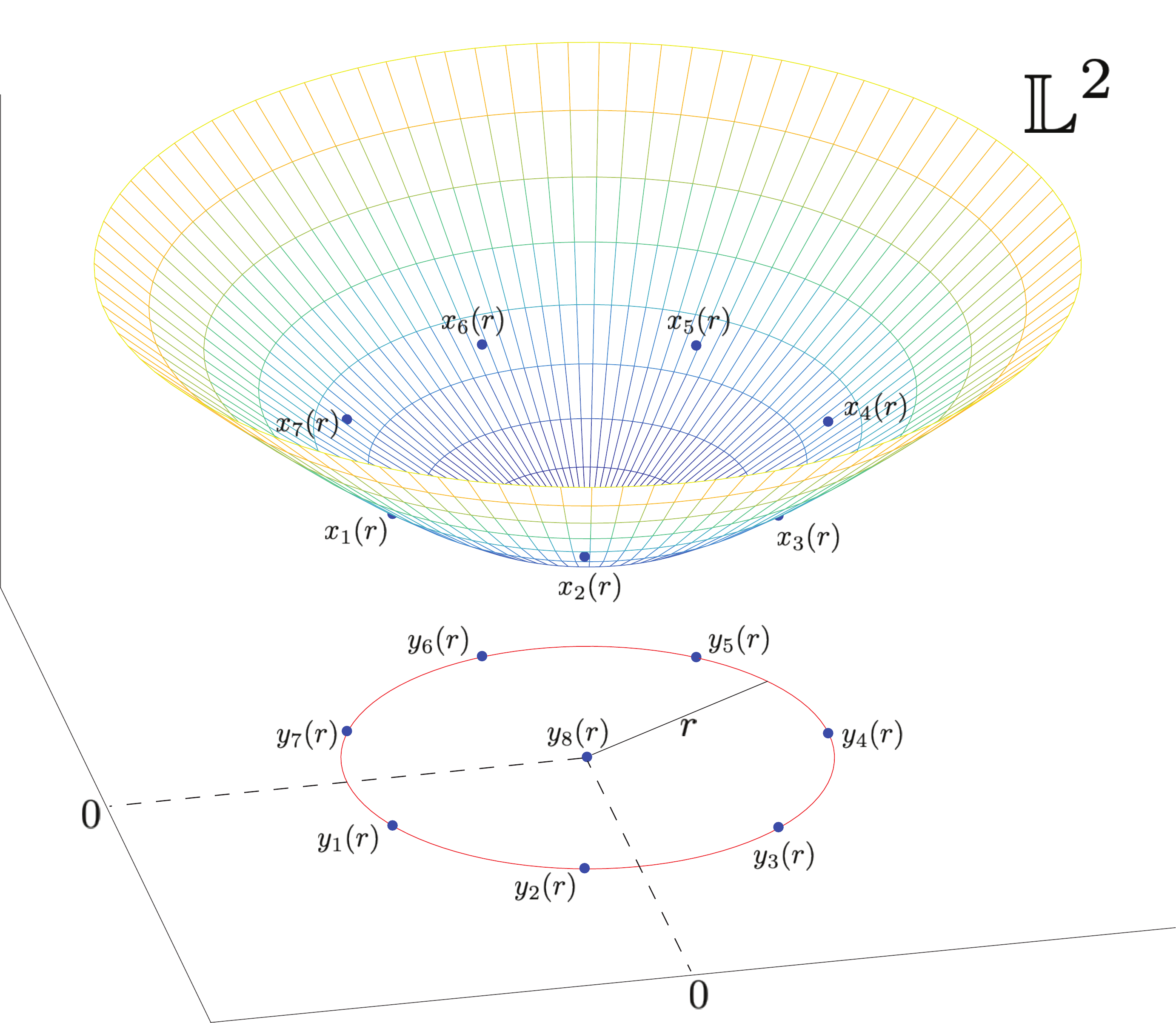}
  \caption{An example of $N=8$ parameterized points $\set{x_n(r)}_{n=1}^{N}$ in $\bL^2$ and  $\set{y_n(r)}_{n=1}^{N}$ in $\R^2$. }
  \label{fig:maximum_spread_hyperbolic}
\end{figure}
For these data points, we have
\begin{align*}
d_{\min} \left(\set{x_n(r)}_{n=1}^{N-1} \right) &= \acosh \left( 1+r^2(1-\cos \frac{2\pi}{N-1} ) \right), \\
d_{\max} \left(\set{x_n(r)}_{n=1}^{N-1}, x_N(r) \right)&= \acosh \left( \sqrt{1+r^2} \right).
\end{align*}
Therefore, for any $N \in \N$, there exists a $r \in \R^{+}$ such that $\set{x_n(r)}_{n=1}^{N} \sqsubseteq \bL^d$. Hence,
\begin{align*}
K(\bL^d) &= \sup \set{ N : \set{x_n(r)}_{n=1}^{N}  \sqsubseteq \bL^d } \\
&= \infty.
\end{align*}
This result hold for all dimensions $d \geq 2$.
\subsubsection{Euclidean space}
\begin{lemma}\label{lem:ord_dense_sets_R}
There is a set of points $x_1, \ldots,x_{N}$ in $\R^d$ such that 
\begin{equation*}
\forall n \in [N-1]: \norm{x_n-x_N} = 1,
\end{equation*}
where $d_{\max} (x_N, \set{x_n}_{n=1}^{N-1}) \leq d_{\min} (\set{x_n}_{n=1}^{N-1})$ and $N = K(\R^d)$.
\end{lemma}
\begin{proof}
Let $\set{y_n}_{n=1}^{N}$ be a set of points in $\R^d$ such that 
\begin{equation*}
d_{\max} (y_N, \set{y_n}_{n=1}^{N-1}) \leq d_{\min} (\set{y_n}_{n=1}^{N-1}),
\end{equation*}
or $\alpha_N(\set{y_n}_{n=1}^{N}) = {N-1 \choose 2}+1$. Without loss of generality, we assume $y_N = 0$ and $d_{\max} (y_N, \set{y_n}_{n=1}^{N-1}) = 1$. Let $x_n = \frac{1}{\norm{y_n} }y_n, \ \forall n \in [N-1]$ and $x_N = y_N$. We want to show that $\alpha_N(\set{x_n}_{n=1}^{N}) \geq \alpha_N(\set{y_n}_{n=1}^{N})$. Following the definition of ordinal spread, we have
\begin{align*}
\alpha_{N} \left( \set{x_n}_{n=1}^{N} \right) & \stackrel{\mathrm{(a)}}{\geq}  \card \set{(i,j): d(x_{i},x_{j}) \geq d_{\max}(x_N, \set{x_n}_{n=1}^{N-1}), i, j \in [N-1], i > j } +1, \\
&\stackrel{\mathrm{(b)}}{=} \card  \set{(i,j): d(x_{i},x_{j}) \geq 1, i, j \in [N-1], i > j } +1, \\
&\stackrel{\mathrm{(c)}}{\geq} \card \set{(i,j): d(y_{i},y_{j}) \geq 1, i, j \in [N-1],i > j } +1, \\
&= \alpha_N(\set{y_n}_{n=1}^{N})
\end{align*}
where $\mathrm{(a)}$ holds with equality if $x_N$ appears last in the sorted distance list, i.e., if $x_N = x_{(N)}$, $\mathrm{(b)}$ is due to $d_{\max} (x_N, \set{x_n}_{n=1}^{N-1}) = 1 = d_{\max} (y_N, \set{y_n}_{n=1}^{N-1})$. To prove inequality $\mathrm{(c)}$, let $d(y_i, y_j) \geq 1$ for distinct $i, j \in [N-1]$. Then, 
\begin{align*}
d(y_i, y_j)^2 &= \frac{\norm{y_i}-1}{\norm{y_i}} \left( \norm{y_i-y_j}^2 - \norm{y_j}^2+\norm{y_i} \right) + \norm{\frac{1}{\norm{y_i}}y_i-y_j }^2 \\
&= \frac{d(y_N,y_i)-1}{\norm{y_i}} \left( d(y_i,y_j)^2 - d(y_N,y_j)^2+d(y_N,y_i) \right) + \norm{\frac{1}{\norm{y_i}}y_i-y_j }^2 \\
&\stackrel{\mathrm{(a)}}{\leq} \norm{\frac{1}{\norm{y_i}}y_i-y_j }^2 \\
&\stackrel{\mathrm{(b)}}{\leq} \norm{\frac{1}{\norm{y_i}}y_i-\frac{1}{\norm{y_j}}y_j }^2 \\
& = d(x_i, x_j)^2
\end{align*}
where $\mathrm{(a)}$ follows from $d(y_N,y_i) \leq 1$, $d(y_N,y_j) \leq 1$,  $d(y_i,y_j)^2 \geq 1$, and $\mathrm{(b)}$ follows from the symmetry in the argument. Therefore, we have
\[
\set{(i,j): d(y_{i},y_{j}) \geq 1, i,j \in [N-1],i > j }   \subseteq  \set{(i,j): d(x_{i},x_{j}) \geq 1, i,j \in [N-1], i > j }.
\]
Hence, $\set{x_n}_{n=1}^{N}$ is an ordinally dense subset of $\R^d$.
\end{proof}
From \Cref{lem:ord_dense_sets_R}, we want find an ordinally dense set of points $x_1, \ldots,x_{N}$ in $\R^d$ such that
\[
\forall n \in [N-1]: \norm{x_{n}} = 1, \mbox{ and } x_N = 0.
\]
From the definition of ordinal spread, we have
\begin{align*}
\alpha_N(\set{x_n}_{n=1}^{N}) &= \card  \set{(i,j): d(x_{i},x_{j}) \geq d_{\max}(x_N, \set{x_n}_{n=1}^{N-1}),i,j \in [N-1], i > j } +1,\\
&= \card \set{(i,j): \norm{x_i}^2+\norm{x_j}^2-2x_{i}^\T x_{j} \geq 1^2, i,j \in [N-1], i > j } +1,\\
&= \card \set{(i,j):  \acos (x_{i}^\T x_{j}) \geq \frac{\pi}{3}, i,j \in [N-1],i > j } +1.
\end{align*}
We can find a maximum number of ordinally dense points by solving a spherical cap packing problem; see \Cref{fig:cap_packing_R}. 
 \begin{figure}[b]
\center
  \includegraphics[width=.25 \linewidth]{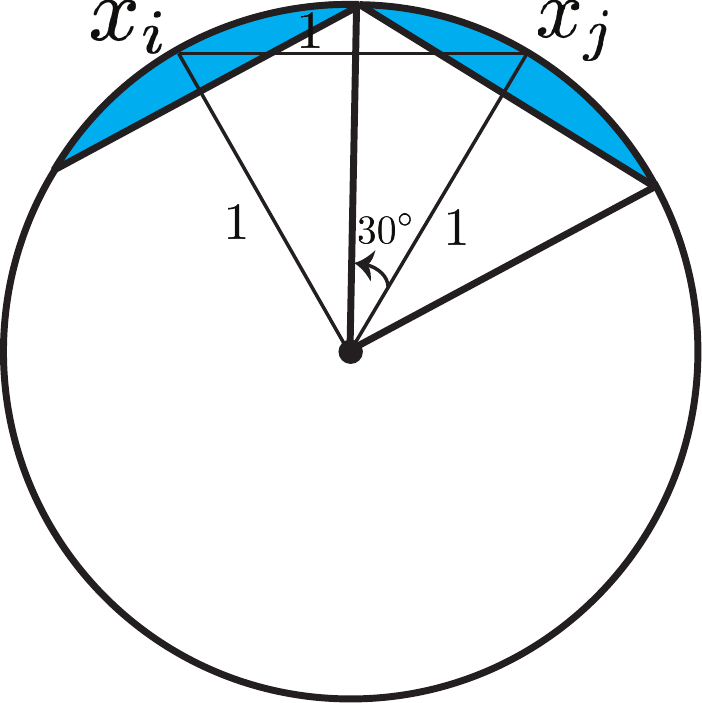}
  \caption{Spherical $\frac{\pi}{6}$-cap packing on the surface of a unit sphere $\S^1$.}
  \label{fig:cap_packing_R}
\end{figure}

\begin{definition}
Let $\S^{d-1}$ be the $(d-1)$-dimensional unit sphere in $\R^d$. We define the spherical $\alpha$-cap $C_x(\alpha)$ as
\begin{equation*}
C_x(\alpha) = \set{y \in \S^{d-1}: x^\T y < \cos(\alpha)},
\end{equation*}
for any $x \in  \S^{d-1}$.
\end{definition}

\begin{definition}
The maximum number of non-overlapping $C_x(\alpha)$ is defined as
\[
N(\alpha) = \max_{N \in \N} \set{N:   \exists x_1, \ldots, x_N \in \S^{d-1} \mbox{such that }\bigcup_{j \in \mathcal{I}, j \neq i} C_{x_j}(\alpha) \cap C_{x_i}(\alpha)  = \emptyset, \forall \mathcal{I} \subseteq [N], \forall i \in [N] }.
\]
\end{definition}
Therefore, we have
\begin{align*}
K(\R^d) &= \sup \set{ \card \set{x_n} :  \set{x_n} \sqsubseteq \R^d},\\
& =\sup \set{ N: x_1, \ldots , x_N \in \R^d, \alpha_N \left( \set{x_n}_{n=1}^{N} \right) = {N-1 \choose 2}+1}, \\
& =\sup \set{ N: x_1, \ldots , x_N \in \R^d, \card  \set{(i,j):  \acos (x_{i}^\T x_{j}) \geq \frac{\pi}{3}, i,j \in [N-1], i > j }= {N-1 \choose 2}},\\
& =\sup \set{ N: x_1, \ldots , x_N \in \R^d\mbox{ such that } \acos (x_{i}^\T x_{j}) \geq \frac{\pi}{3}, i, j \in [N], i \neq j }+1,\\
&\stackrel{\mathrm{(a)}}{=} N(\frac{\pi}{6})+1, \\
&\stackrel{\mathrm{(b)}}{\leq}  \floor*{ \sqrt{\frac{\pi}{8}}\frac{\Gamma\left( \frac{d-1}{2} \right)}{\Gamma \left( \frac{d}{2}\right) \int_{0}^{\frac{\pi}{4}} \sin^{d-2} \theta \left( \cos \theta -\frac{\sqrt{2}}{2} \right) d \theta } }+1,
\end{align*}
where $\mathrm{(a)}$ follows from a simple illustration in \Cref{fig:cap_packing_R}, and $\mathrm{(b)}$ is given in \cite{rankin1955closest}. For large $d$, Rankin provided the following approximation,
\[
N(\alpha) \sim \frac{(\frac{1}{2} \pi d^3 \cos{2\alpha})^{\frac{1}{2}}}{(\sqrt{2} \sin{\alpha})^{d-1} }.
\]
Therefore, we have $N(\frac{\pi}{6}) \sim \sqrt{\pi} d^{\frac{3}{2}} 2^{\frac{d-3}{2}} = C \exp \set{ -d\log{ \frac{\sqrt{2}}{2} } +\frac{3}{2}\log{d} }$ for some constant $C$. The maximum number of non-overlapping spherical caps of half angle $\theta$ which can be placed on the unit sphere in $\R^d$ is not less than $\exp(-d \log\sin 2\theta + o(d))$ \cite{wyner1967random}. Therefore, the lower bound on $N(\frac{\pi}{6})$ is given by $\exp( -d\log \frac{\sqrt{3}}{2} + o(d))$. 

\subsubsection{Spherical space} \label{sec:spherical_space}
\begin{lemma}\label{lem:ord_dense_sets_S}
There is a set of points $x_1, \ldots,x_{N}$ in $\S^d$ such that 
\[
\forall n \in [N-1]: d \left( x_n, x_N \right) = \acos \left( 1-\epsilon \right),
\]
where $d_{\max} (x_N, \set{x_n}_{n=1}^{N-1}) \leq d_{\min} (\set{x_n}_{n=1}^{N-1})$, $N = K(\S^d)$, and for some $\epsilon \geq 0$.
\end{lemma}
\begin{proof}
Let $\set{y_n}_{n=1}^{N}$ be a set of points in $\S^d$ such that 
\[
d_{\max} (y_N, \set{y_n}_{n=1}^{N-1}) \leq d_{\min} (\set{y_n}_{n=1}^{N-1}),
\]
or $\alpha_N(\set{y_n}_{n=1}^{N}) = {N-1 \choose 2}+1$. Without loss of generality, we assume $\alpha_N(\set{y_n}_{n=1}^{N}) = {N-1 \choose 2}+1$, $y_N = e_1$, \footnote{$e_1$ is the first standard base vector for $\R^{d+1}$.} and $d_{\max} (y_N, \set{y_n}_{n=1}^{N-1}) = \acos \left( 1-\epsilon \right)$. From the latter condition, we have 
\[
y_n \bydef 
\begin{bmatrix} 
\sqrt{1-\norm{z_n}^2}  \\
z_n  
\end{bmatrix}, \mbox{ such that } \norm{z_n} \leq \sqrt{1-(1-\epsilon)^2}.
\]
Let us define
\[
\forall n \in [N-1]: x_n = \begin{bmatrix} 
1-\epsilon \\
\sqrt{1-(1-\epsilon)^2} \frac{1}{\norm{z_n}}z_n
\end{bmatrix}, 
\]
and $x_N = e_1$. Then, we claim $\alpha_N(\set{x_n}_{n=1}^{N}) \geq \alpha_N(\set{y_n}_{n=1}^{N})$. Following the definition of ordinal spread, we have
\begin{align*}
\alpha_N(\set{x_n}_{n=1}^{N}) &\stackrel{\mathrm{(a)}}{=} \card \set{(i,j): d(x_{i},x_{j}) \geq d_{\max}(x_N, \set{x_n}_{n=1}^{N-1}), i,j \in [N-1], i > j } +1, \\
&\stackrel{\mathrm{(b)}}{=} \card  \set{(i,j): d(x_{i},x_{j}) \geq \acos \left( 1-\epsilon \right), i,j \in [N-1], i > j } +1, \\
&\stackrel{\mathrm{(c)}}{\geq} \card \set{(i,j): d(y_{i},y_{j}) \geq \acos \left( 1-\epsilon \right), i,j \in [N-1], i > j } +1, \\
&= \alpha_N(\set{y_n}_{n=1}^{N}),
\end{align*}
where $\mathrm{(a)}$ holds with equality if $x_N$ appears last in the sorted distance list, $\mathrm{(b)}$ is due to $d_{\max} (x_N, \set{x_n}_{n=1}^{N-1}) = \acos \left( 1-\epsilon \right) = d_{\max} (y_N, \set{y_n}_{n=1}^{N-1})$. For inequality $\mathrm{(c)}$, let $d(y_i, y_j) \geq \acos \left( 1-\epsilon \right)$ for distinct $i, j \in [N-1]$ and $z_i^\T z_j = \norm{z_i} \norm{z_j} \cos \theta_{ij}$. Therefore,
\begin{align*}
\cos \theta_{ij} &= \frac{1}{\norm{z_i} \norm{z_j} } z_i^\T z_j  \\
&\stackrel{\mathrm{(a)}}{\leq} \frac{1}{\norm{z_i} \norm{z_j} } \left( 1-\epsilon - \sqrt{1-\norm{z_i}^2} \sqrt{1-\norm{z_j}^2} \right) \\
&\stackrel{\mathrm{(b)}}{\leq} 0.
\end{align*}
where $\mathrm{(a)}$ is due to 
\[
y_i^{\T}y_j = \sqrt{1-\norm{z_i}^2} \sqrt{1-\norm{z_j}^2}+z_i^\T z_j \leq 1-\epsilon,
\]
and inequality $\mathrm{(b)}$ is due $ \sqrt{1-\norm{z_i}^2}  \geq \sqrt{1-\sqrt{1-(1-\epsilon)^2}^2} =\sqrt{1-\epsilon^2}$. \footnote{Similarly, we have  $\sqrt{1-\norm{z_j}^2}  \geq \sqrt{1-\epsilon^2}$.}  Then, we have
\begin{align*}
d(x_i,x_j) &= \acos \left( (1-\epsilon)^2 +(1-(1-\epsilon)^2)\cos \theta_{ij} \right) \\
&\geq \acos \left( \sqrt{1-\norm{z_i}^2}\sqrt{1-\norm{z_j}^2} + z_i^\T z_j \right) \\
&= d(y_i,y_j)
\end{align*}
since $(1-(1-\epsilon)^2)\cos \theta_{ij} \leq \norm{z_i}\norm{z_j} \cos \theta_{ij}$ where $\cos \theta_{ij} \leq 0$. Therefore, we have
\[
\set{(i,j): d(y_{i},y_{j}) \geq \acos \left( 1-\epsilon \right), i,j \in [N-1], i > j }   \subseteq   \set{(i,j): d(x_{i},x_{j}) \geq \acos \left( 1-\epsilon \right), i,j \in [N-1], i > j }.
\]
Hence, $\set{x_n}_{n=1}^{N}$ is an ordinally dense subset of $\S^d$.
\end{proof}
Now, let us find ordinally dense set of points $x_1, \ldots,x_{N}$ in $\S^d$ with 
\[
\forall n \in [N-1]: x_n = \begin{bmatrix} 
1-\epsilon  \\
z_n  
\end{bmatrix},  \mbox{ and } x_N = e_1.
\]
We have $\norm{z_n}^2 = 1- (1-\epsilon)^2$ for all $\forall n \in [N-1]$. We begin from the definition of ordinal spread as follows
\begin{align*}
\alpha_N(\set{x_n}_{n=1}^{N}) &= \card \set{(i,j): d(x_{i},x_{j}) \geq d_{\max}(x_N, \set{x_n}_{n=1}^{N-1}), i,j \in [N-1], i > j } +1,\\
&= \card  \set{(i,j): d(x_{i}, x_{j}) \geq \acos(1-\epsilon), i,j \in [N-1], i > j } +1,\\
&= \card \set{(i,j): \frac{1}{\norm{z_i}\norm{z_j}}z_{i}^\T z_{j} \leq \frac{\epsilon(1-\epsilon)}{1-(1-\epsilon)^2}, i,j \in [N-1], i > j } +1,\\
&= \card \set{(i,j):  \acos (\widehat{z}_i^\T \widehat{z}_j)\geq \frac{\pi}{3}, i,j \in [N-1], i > j } +1,\\
\end{align*}
where $\widehat{z}_i = \frac{1}{\norm{z_i}}z_i$, $\widehat{z}_j = \frac{1}{\norm{z_j}}z_j$, and $\sup_{\epsilon}  \frac{\epsilon(1-\epsilon)}{1-(1-\epsilon)^2} = \frac{1}{2}$. Similar to the Euclidean space, this problem is equivalent to spherical $\frac{\pi}{6}$-cap packing number in $\R^{d}$, since $\widehat{z}_n \in \R^{d}$. Finally, if  we assume $\min_{i,j \in [N], i>j} d(x_i,x_j) = \delta$, we have $d_{\max}(x_N, \set{x_n}_{n=1}^{N-1}) \geq \delta$. Therefore, the cap angles can be computed as follows
\[
\alpha = \min_{ \epsilon \geq 1-\cos \delta } \frac{1}{2} \acos  \frac{\epsilon(1-\epsilon)}{1-(1-\epsilon)^2} = \frac{1}{2} \acos \frac{\cos \delta }{1+\cos \delta } > \frac{\pi}{6}.
\]
In this case, the ordinal capacity can be refined as spherical $\alpha$-cap packing number.
\subsection{Proof of Theorem 2} \label{sec:space_form_n_point_curvature}
Let $S$ be a $d$-dimensional space form, and $N \leq K(S)$. From Definition 4, we can find an ordinally dense subset $x_1, \ldots, x_N \in S$. Hence, we have
\begin{align*}
A_N(S) &= \sup_{x_1, \ldots, x_N \in S} \alpha_{N} \left( \set{x_n}_{n=1}^{N} \right), \\
&\stackrel{\mathrm{(a)}}{=} {N-1 \choose 2}+1
\end{align*}
where $\mathrm{(a)}$ directly follows from Proposition 1. This is the number of edges of a complete graph with $N-1$ vertices plus one.

Now, let us consider $N > K(S)$. This could only happen in ($d$-dimensional) Euclidean and spherical spaces, since hyperbolic spaces have infinite ordinal capacity, i.e., $K(\bH^d) = \infty$. 

In \Cref{sec:ordinal_capacity_number}, we proved that there is a set of points $x_1, \ldots, x_{N-1} \in \R^d$ on the unit sphere and $x_N = 0$ such that
\begin{align*}
A_N(S) &= \alpha_N \left( \set{x_n}_{n=1}^N \right), \\
&= \card \set{(i,j): d(x_i,x_j) \geq 1, i,j \in [N-1], i>j}+1.
\end{align*}

Consider a pair of points $x_i , x_j \in \R^d$ with $d(x_i , x_j) < 1$. We can move the point $x_i$ and place it on $x_j$ if 
\[
\card \set{(i,k): d(x_i,x_k) \geq 1, i,k \in [N-1], i \neq k} \leq \card \set{(j,k): d(x_j,x_k) \geq 1, j,k \in [N-1], j \neq k}.
\]
This condition is to ensure that we do not decrease $\alpha_N \left( \set{x_n}_{n=1}^N \right)$. We repeat this process and lump the set of $N-1$ point on $K < N-1$ positions, i.e., $p_1, \ldots, p_K$. At each position $p_k$, we place multiple vertices. Finally, $\alpha_N \left( \set{x_n}_{n=1}^N \right)$ is equal to the number of edges -- with length greater than $1$ -- in this $K$-partite graph with $N-1$ vertices. This graph is $K$-partite because the distance between points in a partition have distances of zero. Hence, their edges do not contribute in calculating the ordinal spread of the point set. This graph becomes a complete $K$-partite graph if all distinct positions $\set{p_k}$ belong to the centers of spherical $\frac{\pi}{6}$-caps on the unit sphere. On the other hand, the number of edges in a complete $K$-partite graph is maximized when the size of the parts differs by at most one, i.e., Tur\'an graph $T(N-1,K)$ \cite{turan1941external}. Therefore, the $N$-point ordinal spread of $S$ (Euclidean or spherical space) is given by
\[
A_N (S) = \card \ E (T(N-1, K(S)-1))+1.
\]
The maximum number of possible partitions ($K(S)-1$) gives the maximum number of edges, i.e.,
\[
\card \ E (T(N-1,1)) \leq \card \ E (T(N-1,2))\leq  \cdots \leq \card\  E(T(N-1,K(S)-1)).
\]
This completes the proof.
\subsection{Proof of Proposition 2}
The proof follows from the definition of $A_N(S)$, the $N$-point ordinal spread of a space form $S$, in Theorem~2.
\section{NUMERICAL EXPERIMENTS}
All our experiments were conducted on a Dual-Core Intel Core i5 Mac machine, 16GB of system memory.
\subsection{Datasets}
In this paper, we used cartographic data (counties in the state of Illinois, counties in Midwestern states, \footnote{The US Zip Code Latitude and Longitude data is available at \href{https://public.opendatasoft.com/explore/dataset/us-zip-code-latitude-and-longitude/}{https://public.opendatasoft.com/explore/dataset/us-zip-code-latitude-and-longitude/}).} and cities and towns across the world \footnote{World cities database is available at \href{https://simplemaps.com/data/world-cities}{https://simplemaps.com/data/world-cities}.}) and single-cell RNA expression data \footnote{We used four data sets: \href{https://www.10xgenomics.com/resources/datasets/hodgkins-lymphoma-dissociated-tumor-targeted-compare-immunology-panel-3-1-standard}{Lymphoma patient}, \href{https://www.10xgenomics.com/resources/datasets/pbm-cs-from-a-healthy-donor-targeted-immunology-panel-3-1-standard}{healthy donor}, \href{https://www.nature.com/articles/ncomms14049}{blood cells landmark}, and \href{https://shiny.mdc-berlin.de/psca/}{Planaria single cell atlas}. }  \cite{zheng2017massively,Hodgkin2020Lymphoma,PBMCs2020HealthyDonor,plass2018cell} which are publicly available datasets. Details of the single-cell expressions datasets are as follows:
\begin{enumerate}
\item \emph{Lymphoma patient.} Human dissociated lymph node tumor cells of a 19-year-old male Hodgkin’s Lymphoma patient were obtained by 10x Genomics from Discovery Life Sciences. Whole transcriptome libraries were generated with Chromium Next GEM Single Cell 3' Reagent Kits v3.1 (Dual Index) User Guide (CG000315) and sequenced on an Illumina NovaSeq 6000. The targeted libraries were generated using the Targeted Gene Expression Reagent Kits User Guide (CG000293) and Human Immunology Panel reagent (PN-1000246) and sequenced on an Illumina NovaSeq 6000.
\item \emph{Lymphoma-healthy donor.}  Human peripheral blood mononuclear cells (PBMCs) of a healthy female donor aged 25 were obtained by 10x Genomics from AllCells. Whole transcriptome libraries were generated with Chromium Next GEM Single Cell 3' Reagent Kits v3.1 (Dual Index) User Guide (CG000315) and sequenced on an Illumina NovaSeq 6000.  The aforementioned two datasets have 13410 samples (combined) and each for a class (binary classification). The dimension of each cell expression vector is 1020.
\item \emph{Blood cells landmark.}  We use the dataset originally from this paper and extract the gene expression data for (1) B cells, (2) Memory T cells, and (3) Native T cells. The complete dataset has $94655$ samples, and the dimension of each cell expression vector is $965$.
\item We use the single-cell RNA sequencing atlas provided in \cite{plass2018cell}. This atlas contains $26000$ cell expression vectors for adult planarians. Each cell is a $21000$-dimensional integer-valued vector representing read counts of gene expressions. Therefore, this raw data reside in a $21000$-dimensional Euclidean space.
\end{enumerate}
{\bf Imputations.} Existing methods for denoising and imputation of raw scRNA-seq data often involve building connection graphs among cells~\cite{li2018accurate, van2017magic} using the distance between cells to diffuse the expression profiles among neighbor cells and smooth out possible outliers. In our experiment we used MAGIC~\cite{van2017magic} to impute our raw sequencing data with different number of neighbors and steps in the diffusion process to get different level of imputation results. 

{\bf RFA score.} For datasets $(1-3)$, we construct the $5$-nearest neighbbor graph, and set the kernel width ($\sigma$) to have an (soft) average of $3$ neighbors; see Section 2.3 for more detail on computing RFA scores.

\subsection{Hyperbolicity of Trees}
 \begin{figure}[b]
\center
  \includegraphics[width=1 \linewidth]{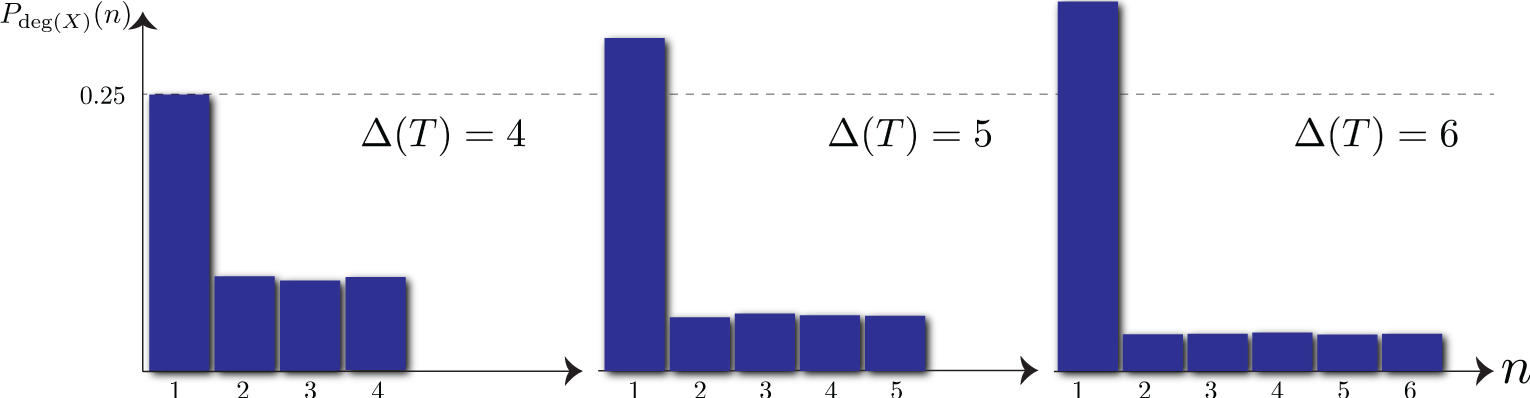}
  \caption{The distribution of node degrees for each random tree.}
  \label{fig:node_degree}
\end{figure}
We generate random weighted trees with $N = 10^4$ nodes. The edge weights are drawn from i.i.d. uniform distribution in $[0,1]$. The distance between each two nodes is the weight of the path joining them. We contaminate the corresponding distance matrix by an additive zero mean Gaussian noise with the signal to noise ratio of $40 \ \mathrm{dB}$. In this experiment, we consider three different trees with maximum degrees of $4,5,6$. \footnote{In the main manuscript, we only considered a binary tree with $\Delta(T) = 3$.} In \Cref{fig:node_degree}, we show the distribution of node degrees for each tree.

We generate random points in space forms of dimension $d = 2, \ldots, 5$, from the following distributions
\begin{itemize}
\item Hyperbolic space: $x = [\sqrt{1+\norm{z}^2}, z^\T]^\T$, where $z \sim \mathcal{N}(0,\sigma^2 I)$ and $\sigma = 100$;
\item Euclidean space: $x \sim \mathcal{N}(0,\sigma^2 I)$; 
\item Spherical space: $x = \frac{1}{\norm{z}} z$, where $z \sim \mathcal{N}(0, I)$. \footnote{Therefore, the points are distributed uniformly on $\bS^d$.}
\end{itemize}

Commonly, in embedding trees, the leaves concentrate near the boundary of the Poincar\'e disk. Hence, we choose a large variance $\sigma$ to heavily sample the points closer to the boundary of Poincar\'e disk. Finally, we devise a hypothesis test based on the total variation distance of probability measures, \footnote{The total variation distance is $\delta(P,Q) =\frac{1}{2}\norm{P-Q}_{1}$, but we can ignore the constant term. } i.e.,
\[
\delta(P,Q) =\norm{P-Q}_{1}.
\]
For each tree $T$ with $\Delta(T) = 4,5$ and $6$, we report the distances between the target (oracle) and empirical probability mass functions (PMF) of $\alpha_{k}$ for a set of $N$ points generated in each space form. In \Cref{tab:tab1,tab:tab2,tab:tab3}, we consider sub-cliques --- randomly sampled from each tree --- with $N=20$ nodes. From the hypothesis tests for $\alpha_{k}$, $k \in \set{3, \ldots, 20}$, we conclude that the ordinal spread variables of random trees better match with hyperbolic ordinal spread variables.
\begin{table}[H]  
\caption{$\delta( P_{\alpha_{k}}, \widehat{P}_{\alpha_{k}}) \times 10^{-3}$ for different space forms --- $\Delta(T) = 4$.}
\centering
\setlength{\tabcolsep}{3pt}
\begin{tabular}{l cccccccccccccccccc} 
\toprule
$k$ 		&$3$&$4$&$5$&$6$&$7$&$8$&$9$ & $10$ &	 $11$ & $12$ &$13$& $14$ &$15$& $16$ &$17$& $18$ &$19$ & $20$\\ \hline
$\bH^2$&$0$&$1.6$&\cellcolor{Gray}$1.9$&\cellcolor{Gray}$1.7$&\cellcolor{Gray}$1.8$&\cellcolor{Gray}$2.1$&\cellcolor{Gray}$2.4$&\cellcolor{Gray}$2.5$&\cellcolor{Gray}$2.3$&\cellcolor{Gray}$2.1$&\cellcolor{Gray}$2.1$& \cellcolor{Gray}$2.0$&\cellcolor{Gray}$2.0$&\cellcolor{Gray}$2.1$&\cellcolor{Gray}$2.2$& \cellcolor{Gray}$2.3$&$2.4$&$2.2$\\
$\bH^3$&$0$&$2.3$&$2.8$&$2.5$&$2.4$&$2.5$&$2.9$&$3.0$&$2.9$&$2.6$&$2.5$& $2.3$&$2.2$&$2.0$ &$1.9$&$1.8$&\cellcolor{Gray}$1.8$&$1.7$\\ 
$\bH^4$&$0$&$2.7$&$3.3$&$3.0$&$2.8$&$2.8$&$3.0$&$3.2$&$3.3$&$3.0$&$2.8$& $2.6$&$2.4$&$2.2$&$2.0$&$1.7$&$1.5$&\cellcolor{Gray}$1.3$\\ 
$\bH^5$&$0$&$3.0$&$3.6$&$3.4$&$3.0$&$3.1$&$3.2$&$3.4$&$3.5$&$3.3$&$3.0$& $2.8$&$2.6$&$2.3$&$2.1$&$1.8$&$1.5$&$1.1$\\
$\bE^2$&$0$&\cellcolor{Gray}$1.5$&$1.9$&$2.1$&$2.4$&$2.8$&$2.9$&$3.1$&$3.2$&$3.3$&$3.4$& $3.5$&$3.7$&$3.9$&$4.2$&$4.7$&$5.4$ &$6.7$\\  
$\bE^3$&$0$&$2.2$&$2.7$&$2.6$&$2.8$&$3.4$&$3.5$&$3.7$&$3.7$&$3.8$&$3.9$& $4.0$&$4.1$&$4.1$&$4.3$&$4.5$&$4.9$&$5.9$\\  
$\bE^4$&$0$&$2.6$&$3.2$&$3.1$&$3.1$&$3.6$&$3.8$&$3.9$&$4.0$&$4.1$&$4.2$&$4.3$& $4.3$&$4.3$&$4.3$&$4.4$&$4.7$&$5.4$\\  
$\bE^5$&$0$&$2.8$&$3.5$&$3.5$&$3.3$&$3.8$&$3.9$&$4.1$&$4.1$&$4.2$&$4.4$& $4.4$&$4.4$&$4.3$&$4.3$&$4.4$&$4.6$&$5.1$\\ 
$\bS^2$&$0$&$8.4$&$9.9$&$10.2$&$10.1$&$10$&$9.9$&$9.8$&$9.6$&$9.2$&$8.7$&$8.6$&$8.6$&$8.6$&$8.5$&$8.5$&$8.5$&$8.8$\\  
$\bS^3$&$0$&$8.4$&$9.8$&$10.2$&$10.1$&$10.1$&$10$&$9.9$&$9.6$&$9.3$&$8.8$&$8.7$&$8.7$&$8.7$&$8.7$&$8.7$&$8.7$&$8.8$\\  
$\bS^4$&$0$&$8.4$&$9.8$&$10.1$&$10.2$&$10.1$&$10$&$9.9$&$9.7$&$9.4$&$8.9$&$8.8$&$8.8$&$8.8$&$8.8$&$8.8$&$8.8$&$8.8$\\ 
$\bS^5$&$0$&$8.5$&$9.8$&$10.1$&$10.2$&$10.1$&$10.1$&$9.9$&$9.7$&$9.4$&$8.9$&$8.9$&$8.9$&$8.9$&$8.8$&$8.8$&$8.8$&$8.7$\\  \hline
\bottomrule
\end{tabular}
 \label{tab:tab1}
\end{table}
\begin{table}[h]  
\caption{$\delta( P_{\alpha_{k}}, \widehat{P}_{\alpha_{k}})\times 10^{-3}$ for different space forms --- $\Delta(T) = 5$.}
\centering
\setlength{\tabcolsep}{3pt}
\begin{tabular}{l cccccccccccccccccc} 
\toprule
$k$ 		&$3$&$4$&$5$&$6$&$7$&$8$&$9$ & $10$ &	 $11$ & $12$ &$13$& $14$ &$15$& $16$ &$17$& $18$ &$19$ & $20$\\ \hline
$\bH^2$&$0$&$1.6$&\cellcolor{Gray}$2.0$&\cellcolor{Gray}$2.6$&\cellcolor{Gray}$2.8$&\cellcolor{Gray}$2,8$&\cellcolor{Gray}$2.8$&\cellcolor{Gray}$2.6$&\cellcolor{Gray}$2.6$&\cellcolor{Gray}$2.5$&\cellcolor{Gray}$2.4$& \cellcolor{Gray}$2.2$&\cellcolor{Gray} $2.1$&\cellcolor{Gray}$1.9$&\cellcolor{Gray} $1.7$& \cellcolor{Gray}$1.4$&\cellcolor{Gray}$1.1$&\cellcolor{Gray}$0.7$\\
$\bH^3$&$0$&$2.3$&$2.8$&$3.1$&$3.4$&$3.5$&$3.5$&$3.5$&$3.3$&$3.3$&$3.1$& $3.0$&$2.8$&$2.6$ &$2.3$&$2.0$&$1.5$&$0.8$\\ 
$\bH^4$&$0$&$2.7$&$3.3$&$3.4$&$3.7$&$3.9$&$3.9$&$3.9$&$3.8$&$3.7$&$3.6$& $3.5$&$3.3$&$3.1$&$2.8$&$2.5$&$2.0$&$1.2$\\ 
$\bH^5$&$0$&$3.0$&$3.6$&$3.6$&$3.8$&$4.2$&$4.2$&$4.2$&$4.2$&$4.0$&$3.9$& $3.8$&$3.6$&$3.4$&$3.1$&$2.8$&$2.3$&$1.6$\\
$\bE^2$&$0$&\cellcolor{Gray}$1.5$&$2.3$&$3.4$&$3.6$&$4.0$&$4.2$&$4.4$&$4.6$&$4.8$&$5.0$&$5.3$&$5.6$&$5.9$&$6.3$&$6.9$&$7.5$ &$8.4$\\  
$\bE^3$&$0$&$2.2$&$2.7$&$3.8$&$4.2$&$4.6$&$4.8$&$5.0$&$5.1$&$5.2$&$5.4$& $5.6$ &$5.8$&$6.0$&$6.3$&$6.7$&$7.1$&$7.8$\\  
$\bE^4$&$0$&$2.6$&$3.2$&$3.9$&$4.5$&$4.9$&$5.0$&$5.2$&$5.4$&$5.5$&$5.6$& $5.8$&$5.9$&$6.1$&$6.3$&$6.6$&$6.9$&$7.3$\\  
$\bE^5$&$0$&$2.8$&$3.5$&$4.1$&$4.6$&$5.0$&$5.2$&$5.4$&$5.5$&$5.7$&$5.8$& $5.9$&$6.0$&$6.1$&$6.3$&$6.5$&$6.7$&$7.0$\\ 
$\bS^2$&$0$&$8.4$&$9.9$&$10.2$&$10.1$&$10$&$9.9$&$9.8$&$9.6$&$9.5$&$9.5$&$9.5$&$9.5$&$9.5$&$9.5$&$9.5$&$9.6$&$9.7$\\  
$\bS^3$&$0$&$8.4$&$9.8$&$10.2$&$10.1$&$10.1$&$10$&$9.9$&$9.6$&$9.5$&$9.6$&$9.6$&$9.6$&$9.6$&$9.7$&$9.7$&$9.7$&$9.7$\\  
$\bS^4$&$0$&$8.4$&$9.8$&$10.1$&$10.2$&$10.1$&$10$&$9.9$&$9.7$&$9.6$&$9.6$&$9.7$&$9.7$&$9.7$&$9.7$&$9.8$&$9.8$&$9.7$\\ 
$\bS^5$&$0$&$8.5$&$9.8$&$10.1$&$10.2$&$10.1$&$10.1$&$9.9$&$9.7$&$9.7$&$9.7$&$9.7$&$9.7$&$9.8$&$9.8$&$9.8$&$9.8$&$9.7$\\  \hline
\bottomrule
\end{tabular}
 \label{tab:tab2}
\end{table} 
\begin{table}[h] 
\caption{$\delta( P_{\alpha_{k}}, \widehat{P}_{\alpha_{k}})\times 10^{-3}$ for different space forms --- $\Delta(T) = 6$.}
\centering
\setlength{\tabcolsep}{3pt}
\begin{tabular}{l cccccccccccccccccc} 
\toprule
$k$ 		&$3$&$4$&$5$&$6$&$7$&$8$&$9$ & $10$ &	 $11$ & $12$ &$13$& $14$ &$15$& $16$ &$17$& $18$ &$19$ & $20$\\ \hline
$\bH^2$&$0$&$1.6$&\cellcolor{Gray}$2.5$&\cellcolor{Gray}$3.0$&\cellcolor{Gray}$3.2$&\cellcolor{Gray}$3.2$&\cellcolor{Gray}$3.1$&\cellcolor{Gray} 	  $2.9$&\cellcolor{Gray}$2.9$&\cellcolor{Gray}$2.7$&\cellcolor{Gray}$2.5$& \cellcolor{Gray}$2.3$&\cellcolor{Gray} $2.1$&\cellcolor{Gray}$1.9$&\cellcolor{Gray} $1.6$& \cellcolor{Gray}$1.3$&$1.1$&$1.4$\\
$\bH^3$&$0$&$2.3$&$2.8$&$3.6$&$3.8$&$3.8$&$3.8$&$3.7$&$3.6$&$3.5$&$3.3$& $3.0$&$2.8$&$2.5$ &$2.1$&$1.6$&\cellcolor{Gray} $1.0$&$0.8$\\ 
$\bH^4$&$0$&$2.7$&$3.3$&$3.9$&$4.0$&$4.2$&$4.2$&$4.2$&$4.0$&$3.9$&$3.7$& $3.5$&$3.2$&$2.9$&$2.5$&$1.9$&$1.2$&\cellcolor{Gray}$0.4$\\ 
$\bH^5$&$0$&$3.0$&$3.6$&$4.0$&$4.2$&$4.5$&$4.5$&$4.4$&$4.3$&$4.1$&$4.0$& $3.8$&$3.5$&$3.2$&$2.8$&$2.2$&$1.5$&$0.5$\\
$\bE^2$&$0$&\cellcolor{Gray}$1.5$&$2.8$&$3.8$&$4.0$&$4.3$&$4.5$&$4.6$&$4.7$&$4.9$&$5.0$& $5.2$&$5.4$&$5.6$&$5.8$&$6.2$&$6.7$ &$7.5$\\  
$\bE^3$&$0$&$2.2$&$3.0$&$4.2$&$4.5$&$4.9$&$5.0$&$5.2$&$5.3$&$5.3$&$5.4$& $5.5$ &$5.6$&$5.7$&$5.8$&$6.0$&$6.2$&$6.8$\\  
$\bE^4$&$0$&$2.6$&$3.2$&$4.4$&$4.8$&$5.2$&$5.3$&$5.4$&$5.5$&$5.6$&$5.6$& $5.7$&$5.7$&$5.8$&$5.8$&$5.9$&$6.0$&$6.3$\\  
$\bE^5$&$0$&$2.8$&$3.5$&$4.5$&$5.0$&$5.3$&$5.4$&$5.6$&$5.7$&$5.7$&$5.8$& $5.8$&$5.8$&$5.8$&$5.8$&$5.8$&$5.8$&$5.9$\\ 
$\bS^2$&$0$&$8.4$&$9.9$&$10.2$&$10.1$&$10$&$9.9$&$9.8$&$9.6$&$9.4$&$9.4$&$9.4$&$9.4$&$9.3$&$9.3$&$9.3$&$9.2$&$9.3$\\  
$\bS^3$&$0$&$8.4$&$9.8$&$10.2$&$10.1$&$10.1$&$10$&$9.9$&$9.6$&$9.5$&$9.5$&$9.5$&$9.5$&$9.5$&$9.5$&$9.4$&$9.4$&$9.3$\\  
$\bS^4$&$0$&$8.4$&$9.8$&$10.1$&$10.2$&$10.1$&$10$&$9.9$&$9.7$&$9.6$&$9.6$&$9.6$&$9.6$&$9.5$&$9.5$&$9.5$&$9.4$&$9.3$\\ 
$\bS^5$&$0$&$8.5$&$9.8$&$10.1$&$10.2$&$10.1$&$10.1$&$9.9$&$9.7$&$9.6$&$9.6$&$9.6$&$9.6$&$9.6$&$9.6$&$9.5$&$9.5$&$9.3$\\  \hline
\bottomrule
\end{tabular}
 \label{tab:tab3}
\end{table}

{\bf Remark.} Ordinal spread variables of larger sub-cliques are more effective in testing for distinguishing the curvature sign and the dimension of space forms. For instance, consider a triangle in a space form. Regardless of the presumed space form, we have
\[
\alpha_{k} = \begin{cases}
1 &\mbox{ if } \ k = 1,\\
1 &\mbox{ if } \ k = 2, \\
2 &\mbox{ if } \ k = 3,
\end{cases} \ \ \ \mbox{with probability } 1.
\]
This is a trivial result from {\bf Proposition 1} and the related discussion in its proof. Therefore, the statistics of $\alpha_{k}$ can not bear any useful information about the geometry of data. This fact, along with the total number of available unique sub-cliques, should be considered for implementing a proper hypothesis test based on a majority vote; see \Cref{tab:tab1,tab:tab2,tab:tab3}.

Note that we can also deign an aggregate hypothesis test based on ${\bm \alpha}_{N}$ by defining the following distance function between $P_{{\bm \alpha}_{N}}$ and $ \widehat{P}_ {{\bm \alpha}_{N} }$, e.g., $\delta \big( P_{{\bm \alpha}_{N}} , \widehat{P}_ {{\bm \alpha}_{N} }\big) = \sum_{k=1}^{N} \delta( P_{\alpha_{k}}, \widehat{P}_{\alpha_{k}})$. This definition involves all ordinal spread variables related to sub-cliques of size $N$, i.e., $\alpha_{k}$. Then, we can perform minimum-distance hypothesis tests for sub-cliques of sizes $N \in \set{5, \ldots 20}$. For each experiment, hyperbolic spaces provide the best matches for ordinal spread variables of each random tree; see \Cref{tab:tab4,tab:tab5,tab:tab6}.
This aggregate hypothesis test proves to more robustly reveal the hyperbolicity of weighted trees, compared to the individual tests based on ordinal spread variable $\alpha_N$.

\begin{table}[H] 
\caption{$\delta \big( P_{{\bm \alpha}_{N}} , \widehat{P}_ {{\bm \alpha}_{N} }\big)\times 10^{-2}$ for different space forms --- $\Delta(T) = 4$.}
\centering
\setlength{\tabcolsep}{3pt}
\begin{tabular}{l cccccccccccccccc} 
\toprule
$N$ 		&$5$&$6$&$7$&$8$&$9$&$10$&$11$&$12$&$13$&$14$&$15$&$16$&$17$&$18$&$19$&$20$\\ \hline
$\bH^2$&$3.3$&\cellcolor{Gray}$3.7$&\cellcolor{Gray}$3.9$&\cellcolor{Gray}$4.1$&\cellcolor{Gray}$4.2$&\cellcolor{Gray}$4.2$&\cellcolor{Gray}$4.1$&\cellcolor{Gray}$4.1$&\cellcolor{Gray}$4.0$&\cellcolor{Gray}$4.0$&\cellcolor{Gray}$3.9$&\cellcolor{Gray}$3.8$&\cellcolor{Gray}$3.8$&\cellcolor{Gray}$3.7$&\cellcolor{Gray}$3.6$&\cellcolor{Gray}$3.6$\\
$\bH^3$&\cellcolor{Gray}$3.0$&$3.8$&$4.2$&$4.6$&$4.7$&$4.7$&$4.7$&$4.6$&$4.6$&$4.5$&$4.4$&$4.4$&$4.3$&$4.2$&$4.1$&$4.0$\\ 
$\bH^4$&$3.4$&$4.1$&$4.7$&$5.0$&$5.2$&$5.2$&$5.1$&$5.1$&$5.0$&$4.9$&$4.8$&$4.8$&$4.7$&$4.5$&$4.4$&$4.4$\\ 
$\bH^5$&$3.7$&$4.4$&$5.0$&$5.3$&$5.5$&$5.5$&$5.5$&$5.4$&$5.3$&$5.3$&$5.2$&$5.1$&$5.0$&$4.9$&$4.8$&$4.7$\\
$\bE^2$&$7.0$&$8.0$&$8.3$&$8.4$&$8.2$&$8.0$&$7.8$&$7.6$&$7.3$&$7.1$&$6.8$&$6.6$&$6.4$&$6.2$&$6.0$&$5.9$\\  
$\bE^3$&$5.9$&$6.9$&$7.9$&$8.1$&$8.0$&$7.9$&$7.9$&$7.8$&$7.6$&$7.5$&$7.3$&$7.1$&$6.9$&$6.7$&$6.6$&$6.4$\\  
$\bE^4$&$5.2$&$6.6$&$7.6$&$7.9$&$7.9$&$7.9$&$8.0$&$8.0$&$7.8$&$7.7$&$7.5$&$7.3$&$7.2$&$7.0$&$6.9$&$6.7$\\  
$\bE^5$&$5.0$&$6.4$&$7.4$&$7.8$&$7.9$&$8.0$&$8.1$&$8.1$&$8.0$&$7.8$&$7.6$&$7.5$&$7.4$&$7.2$&$7.1$&$6.9$\\ 
$\bS^2$&$12.0$&$15.5$&$17.8$&$19.4$&$20.2$&$20.3$&$20.2$&$19.8$&$19.3$&$18.8$&$18.3$&$17.7$&$17.1$&$16.6$&$16.1$&$15.6$\\  
$\bS^3$&$11.8$&$15.5$&$17.8$&$19.5$&$20.2$&$20.4$&$20.3$&$19.9$&$19.5$&$18.9$&$18.4$&$17.9$&$17.3$&$16.8$&$16.2$&$15.7$\\ 
$\bS^4$&$11.7$&$15.4$&$17.8$&$19.5$&$20.1$&$20.4$&$20.3$&$20.0$&$19.6$&$19.0$&$18.5$&$17.9$&$17.4$&$16.9$&$16.3$&$15.8$\\ 
$\bS^5$&$11.6$&$15.4$&$17.8$&$19.4$&$20.1$&$20.4$&$20.4$&$20.0$&$19.6$&$19.1$&$18.6$&$18.0$&$17.5$&$16.9$&$16.4$&$15.9$\\  \hline
\bottomrule
\end{tabular}
 \label{tab:tab4}
\end{table}

\begin{table}[H] 
\caption{$\delta \big( P_{{\bm \alpha}_{N}} , \widehat{P}_ {{\bm \alpha}_{N} }\big)\times 10^{-2}$ for different space forms --- $\Delta(T) = 5$.}
\centering
\setlength{\tabcolsep}{3pt}
\begin{tabular}{l cccccccccccccccc} 
\toprule
$N$ 		&$5$&$6$&$7$&$8$&$9$&$10$&$11$&$12$&$13$&$14$&$15$&$16$&$17$&$18$&$19$&$20$\\ \hline
$\bH^2$&$5.4$&$5	.5$&\cellcolor{Gray}$5.3$&\cellcolor{Gray}$5.4$&\cellcolor{Gray}$5.2$&\cellcolor{Gray}$5.1$&\cellcolor{Gray}$4.9$&\cellcolor{Gray}$4.8$&\cellcolor{Gray}$4.5$&\cellcolor{Gray}$4.4$&\cellcolor{Gray}$4.2$&\cellcolor{Gray}$4.1$&\cellcolor{Gray}$3.9$&\cellcolor{Gray}$3.8$&\cellcolor{Gray}$3.7$&\cellcolor{Gray}$3.6$\\
$\bH^3$&$4.5$&\cellcolor{Gray}$4.8$&$5.4$&$5.9$&$5.9$&$5.9$&$5.8$&$5.7$&$5.5$&$5.4$&$5.3$&$5.2$&$5.1$&$4.9$&$4.8$&$4.7$\\ 
$\bH^4$&\cellcolor{Gray}$4.1$&$5.0$&$5.9$&$6.3$&$6.4$&$6.4$&$6.4$&$6.3$&$6.2$&$6.2$&$6.1$&$5.9$&$5.8$&$5.7$&$5.5$&$5.4$\\ 
$\bH^5$&$4.2$&$5.2$&$6.2$&$6.6$&$6.8$&$6.8$&$6.8$&$6.8$&$6.8$&$6.7$&$6.6$&$6.4$&$6.3$&$6.2$&$6.1$&$5.9$\\
$\bE^2$&$9.3$&$10.8$&$11.5$&$11.5$&$11.3$&$11.1$&$10.9$&$10.7$&$10.3$&$10$&$9.7$&$9.4$&$9.1$&$8.9$&$8.6$&$8.4$\\  
$\bE^3$&$8.0$&$9.6$&$10.6$&$11.0$&$11.0$&$11.1$&$11.0$&$10.8$&$10.5$&$10.3$&$10.0$&$9.7$&$9.5$&$9.3$&$9.0$&$8.8$\\  
$\bE^4$&$7.4$&$9.0$&$10.3$&$10.8$&$10.9$&$11.0$&$11.0$&$10.8$&$10.6$&$10.4$&$10.1$&$9.9$&$9.7$&$9.5$&$9.3$&$9.1$\\  
$\bE^5$&$6.9$&$8.7$&$10.1$&$10.6$&$10.8$&$10.9$&$10.9$&$10.8$&$10.6$&$10.4$&$10.2$&$10.0$&$9.8$&$9.6$&$9.4$&$9.2$\\ 
$\bS^2$&$13.6$&$17.2$&$19.6$&$21.0$&$21.7$&$21.8$&$21.6$&$21.2$&$20.6$&$20.0$&$19.4$&$18.8$&$18.1$&$17.5$&$16.9$&$16.4$\\  
$\bS^3$&$13.4$&$17.3$&$19.8$&$21.1$&$21.8$&$21.9$&$21.8$&$21.3$&$20.8$&$20.2$&$19.6$&$18.9$&$18.3$&$17.7$&$17.1$&$16.5$\\ 
$\bS^4$&$13.3$&$17.3$&$19.8$&$21.1$&$21.8$&$22.0$&$21.8$&$21.4$&$20.9$&$20.3$&$19.7$&$19.0$&$18.4$&$17.8$&$17.1$&$16.6$\\
$\bS^5$&$13.1$&$17.3$&$19.7$&$21.1$&$21.9$&$22.0$&$21.8$&$21.5$&$20.9$&$20.3$&$19.7$&$19.0$&$18.4$&$17.8$&$17.2$&$16.6$\\  \hline
\bottomrule
\end{tabular}
 \label{tab:tab5}
\end{table}

\begin{table}[H] 
\caption{$\delta \big( P_{{\bm \alpha}_{N}} , \widehat{P}_ {{\bm \alpha}_{N} }\big)\times 10^{-2}$ for different space forms --- $\Delta(T) = 6$.}
\centering
\setlength{\tabcolsep}{3pt}
\begin{tabular}{l cccccccccccccccc} 
\toprule
$N$ 		&$5$&$6$&$7$&$8$&$9$&$10$&$11$&$12$&$13$&$14$&$15$&$16$&$17$&$18$&$19$&$20$\\ \hline
$\bH^2$&$5.8$&$5.8$&$5.9$&\cellcolor{Gray}$5.8$&\cellcolor{Gray}$5.7$&\cellcolor{Gray}$5.5$&\cellcolor{Gray}$5.2$&\cellcolor{Gray}$5.0$&\cellcolor{Gray}$4.9$&\cellcolor{Gray}$4.7$&\cellcolor{Gray}$4.6$&\cellcolor{Gray}$4.4$&\cellcolor{Gray}$4.3$&\cellcolor{Gray}$4.2$&\cellcolor{Gray}$4.1$&\cellcolor{Gray}$3.9$\\
$\bH^3$&$4.8$&$5.2$&\cellcolor{Gray}$5.7$&$5.9$&$6.1$&$6.1$&$5.9$&$5.8$&$5.7$&$5.5$&$5.4$&$5.3$&$5.2$&$5.0$&$4.9$&$4.8$\\ 
$\bH^4$&$4.5$&\cellcolor{Gray}$5.0$&$5.9$&$6.4$&$6.6$&$6.5$&$6.4$&$6.3$&$6.3$&$6.2$&$6.0$&$5.9$&$5.8$&$5.6$&$5.5$&$5.4$\\ 
$\bH^5$&\cellcolor{Gray}$4.3$&$5.1$&$6.2$&$6.7$&$6.9$&$6.9$&$6.8$&$6.7$&$6.7$&$6.6$&$6.5$&$6.3$&$6.2$&$6.1$&$6.0$&$5.8$\\
$\bE^2$&$9.6$&$11.1$&$11.7$&$11.7$&$11.4$&$11.1$&$10.8$&$10.5$&$10.2$&$9.9$&$9.6$&$9.3$&$9.0$&$8.7$&$8.5$&$8.3$\\  
$\bE^3$&$8.3$&$9.9$&$10.8$&$11.0$&$10.9$&$10.9$&$10.8$&$10.6$&$10.4$&$10.1$&$9.8$&$9.6$&$9.3$&$9.1$&$8.9$&$8.7$\\  
$\bE^4$&$7.7$&$9.2$&$10.3$&$10.7$&$10.8$&$10.8$&$10.8$&$10.6$&$10.4$&$10.2$&$10.0$&$9.7$&$9.5$&$9.3$&$9.1$&$8.9$\\  
$\bE^5$&$7.2$&$8.8$&$10.0$&$10.5$&$10.6$&$10.8$&$10.8$&$10.6$&$10.4$&$10.2$&$10.0$&$9.8$&$9.6$&$9.4$&$9.2$&$9.0$\\ 
$\bS^2$&$13.3$&$17.0$&$19.3$&$20.6$&$21.3$&$21.4$&$21.2$&$20.8$&$20.3$&$19.7$&$19.1$&$18.5$&$17.9$&$17.3$&$16.7$&$16.2$\\  
$\bS^3$&$13.1$&$17.1$&$19.4$&$20.7$&$21.4$&$21.5$&$21.4$&$20.9$&$20.5$&$19.9$&$19.3$&$18.7$&$18.0$&$17.4$&$16.9$&$16.3$\\ 
$\bS^4$&$13.0$&$17.1$&$19.4$&$20.7$&$21.4$&$21.6$&$21.4$&$21.0$&$20.5$&$20.0$&$19.4$&$18.7$&$18.1$&$17.5$&$16.9$&$16.4$\\ 
$\bS^5$&$12.9$&$17.1$&$19.4$&$20.7$&$21.4$&$21.6$&$21.5$&$21.1$&$20.6$&$20.0$&$19.4$&$18.8$&$18.2$&$17.6$&$17.0$&$16.4$\\  \hline
\bottomrule
\end{tabular}
 \label{tab:tab6}
\end{table}

\subsubsection{On Euclidean Embedding Dimension of Trees}
\begin{table}[b] 
  \centering
    \caption{The $N$-point ordinal spread for $\bE^2, \bE^3,\bE^4$ versus $\widehat{A}_N$ estimated from $\widetilde{D}_4, \widetilde{D}_5$ and $\widetilde{D}_6$. }
\begin{tabular}{cccc cccccc}
\toprule
$N$ & $6$& $8$ & $10$ & $12$ & $14$ & $16$ & $18$ & $20$ & $100$\\
\hline
\rowcolor{Gray}  $\widetilde{D}_{4}: \widehat{A}_N$ &  $11$ & $22$ & $37$ & $56$ & $79$ & $106$ & $137$ & $169$ & $4421$\\
$\widetilde{D}_{5}: \widehat{A}_N$ & $11$ & $22$ & $37$ & $56$ & $79$ & $106$ & $136$ & $172$ & $4412$ \\
\rowcolor{Gray} $\widetilde{D}_{6}: \widehat{A}_N$ & $11$  & $22$ & $37$ & $56$ & $79$ & $106$& $137$ & $170$&  $4454$\\
$A_M(\bE^2)$ & $11$ & $21$ & $34$ & $51$ & $71$ & $94$ & $121$ & $151$ & $4048$ \\
\rowcolor{Gray} $A_N(\bE^3)$ & $11$  & $22$ & $37$ & $56$ & $79$ & $106$& $135$ & $168$&  $4573$\\
$A_N(\bE^4)$ & $11$ & $22$ & $37$ & $56$ & $79$ & $106$ & $137$ & $172$ & $4741$ \\
\bottomrule
\end{tabular}
  \label{tab:tree_table}
\end{table}
We generate a random tree $T$ with $N=10^4$ nodes, maximum degree of $\Delta$, and i.i.d. edge weights from $\mathrm{unif}(0,1)$. Let $\widetilde{D}_\Delta = D_\Delta + n$, where $n$ is a zero mean Gaussian noise with $40$ decibel signal-to-noise ratio, be the noisy distance matrix for $T$. The embedding goal is to find a representation $x_1, \ldots, x_N$ for tree nodes in $S$, such that
\[
d(x_{i}, x_{j}) \leq d( x_{k}, x_{l}) \Longleftrightarrow \widetilde{D}_\Delta (i,j) \leq \widetilde{D}_\Delta (k,l).
\]

We randomly select $10^6$ sub-cliques of sizes $N \in \set{2,4, \ldots, 20, 100}$. In \Cref{tab:tree_table}, we give the empirical $N$-th ordinal spread based on non-metric measurements associated with the sub-cliques, i.e., $\widehat{A}_N$. The distribution-free test gives a lower bound of $\hat{d} \geq 4$ for Euclidean embedding dimension. 

On the other hand, consider a random weighted tree and a node $x_n$ with degree $\Delta_n$. \footnote{We assume the existence of a perfect embedding.} We can easily see that 
\[
\max_{i \in [\Delta_n]} d(x_n , x_{n_i}) \leq \min_{\substack{ i,j \in [\Delta_n] \\ i \neq j}} d(x_{n_i} , x_{n_j}),
\]
where $x_{n_1}, \ldots, x_{n_\Delta}$ are adjacent points to $x_n$. Hence, $\set{x_n} \cup \set{x_{n_i}}_{i=1}^{\Delta_n}$ is a set of $\Delta_n+1$ points with maximum ordinal spread. Therefore, a lower bound for embedding dimension of a metric tree $T$ (in Euclidean space) is given by
\begin{align*}
\hat{d} \geq  \min \set{d: K(\R^d) \geq \Delta (T) +1} 
\end{align*}
The exponential growth of $\rho_d$ gives $\hat{d} = \Omega (\log \Delta(T))$.

{\bf Remark.} In absence of any prior information for proper distributions of data points, the estimate for the dimension of underlying space form is unreliable.
The statistics of the ordinal spread variables are \emph{invariant} with respect isotonic transformation of data points, e.g., rotation, translation, and uniform scaling in Euclidean space. 
\begin{fact}
Let $\set{x_n}_{n=1}^{N}$ be a set of points in $(S,d)$. The ordinal spread vector is invariant with respect to strongly isotonic transformation \cite{kleindessner2014uniqueness} of points. In other words, let $\psi: S \rightarrow S$ be an arbitrary function such that for all $x,y,z,w \in S$ we have
\begin{align*}
d(x,y) < d(z,w) &\Rightarrow d \big( \psi(x), \psi(y) \big) < d\big(\psi(z), \psi(w) \big), \\
d(x,y) = d(z,w) &\Rightarrow d \big( \psi(x), \psi(y) \big) = d\big( \psi(z), \psi(w) \big),
\end{align*}
then, $\boldsymbol{\alpha}  \left( \set{x_n}_{n=1}^{N} \right) = \boldsymbol{\alpha}  \left( \set{\psi(x_n)}_{n=1}^{N} \right)$.
\end{fact}
Therefore, we can also use compact distributions, e.g., the multivariate uniform distribution. The arbitrary choices of Gaussian and uniform distributions do not significantly change the statistics of the ordinal spread variables --- at least, it does not affect the key results in this experiment.

\subsection{Single-cell RNA Expression Data}
We use the single-cell RNA sequencing atlas provided in \cite{plass2018cell}. This atlas contains $26000$ cell expression vectors for adult planarians. Each cell is an integer-valued vector representing read counts of gene expressions. The specific choices of pre-processing method and the comparison criteria \emph{imply} a geometry --- namely, geometry of similarity comparisons --- that is not necessarily related to the domain of data vectors. The choice of comparisons are as follows:
\begin{itemize}
\item {\bf $\ell_2$ distance:} The points $x_i, x_j$ are more similar to each other than $x_k, x_l$ if $\norm{x_i-x_j}_2 \leq \norm{x_k-x_l}_2$. The \emph{true} geometry of comparisons is a $21000$-dimensional Euclidean space;
\item { \bf Angles:} The points $x_i, x_j$ are more similar than $x_k, x_l$ if $\measuredangle (x_i,x_j) \leq \measuredangle (x_i,x_j)$, where
\[
\measuredangle (x_i,x_j) \bydef \acos \frac{(x_i - \mu)^\T (x_j-\mu)}{\norm{x_i-\mu}\norm{x_j-\mu}},
\]
and $\mu = \frac{1}{N} \sum_{n \in [N]} x_n$. We use spherical distance to compare cell vectors. The geometry of comparisons is a spherical space of dimension $21000-1$.
\item { \bf Relative forest accessibility (RFA) index:} For a set of points $x_1 \ldots, x_N$, we construct the local connectivity edge set $E$ from a symmetric $k$-nearest neighbor method. The relative forest accessibility matrix is a $N \times N$ doubly stochastic matrix defined as $P =(I+L)^{-1}$ where $L = D-A$ is the Laplacian matrix, $A = (A_{i,j})$ such that
\[
A_{i,j} = \exp \big(- \frac{\norm{x_i-x_j}^2}{2\sigma^2} \big) [(i,j) \in E],
\]
where the Iverson bracket $[(i,j) \in E] = 1$ if $(i,j) \in E$ and is $0$ otherwise, and $D$ is a diagonal matrix with $D_{ii}= \sum_{j \in [N]}A_{i,j}$. The $ij$-th element of $P$ is the probability of a spanning forest includes a rooted tree at $x_i$ and is connected to $x_j$ --- a measure of similarity between $x_i$ and $x_j$ \cite{klimovskaia2020poincare}. In this experiment, we let $\sigma = \frac{1}{\sqrt{10}N^2}\sum_{i,j \in [N]} \norm{x_i-x_j}$ and ignore the hard edge assignment since the conservative choice of kernel width performs a soft edge assignment; see \Cref{fig:prior} $(b)$. For a fast implementation of $P = (I+L)^{-1}$, we approximate the weighted adjacency matrix $A \in \R^{N \times N}$ with a rank-$500$ semidefinite matrix --- via a simple eigenvalue thresholding --- and use Woodbury matrix identity to compute $P$. The points $x_i, x_j$ are more similar than $x_k, x_l$ if the relative forest accessibility index $p_{i,j}$ is greater than $p_{k,l}$. The geometry of RFA comparisons is unknown.
\end{itemize}

We propose heavy-tailed distributions for oracle distributions. To confirm the validity of this choice, we generate random normal, and log-normal (a heavy-tailed distribution) of various dimensions $d = 2, \ldots, 20$, and different scale parameters; see \Cref{fig:prior} $(a)$. The log-normal distributions gives a better match for the empirical distribution of $\norm{X-\mathbb{E} X}$ --- which we denote as a \emph{side information}. Specifically, we generate points according to the following distributions to find the optimal dimension and scale parameter $a$:
\begin{itemize}
\item Hyperbolic space: $x = [\sqrt{1+\norm{z}^2}, z^\T]^\T$, where $z \sim e^{a\mathcal{N}(0, I)} - \mathbb{E}e^{a\mathcal{N}(0, I)}$;
\item Euclidean space: $x \sim e^{a\mathcal{N}(0, I)}$; 
\item Spherical space: $x = \frac{1}{\norm{z}} z$, where $z \sim e^{a\mathcal{N}(0, I)} - \mathbb{E}e^{a\mathcal{N}(0, I)}$. 
\end{itemize}
Therefore, we approximate the distribution of RNAseq read counts by (projected) log-normal distributions in a space form. In bioinformatics, the heavy-tailed distribution of gene expression reads is a well-known fact  \cite{ono2013pbsim}.
 \begin{figure}[H]
\center
  \includegraphics[width=1\linewidth]{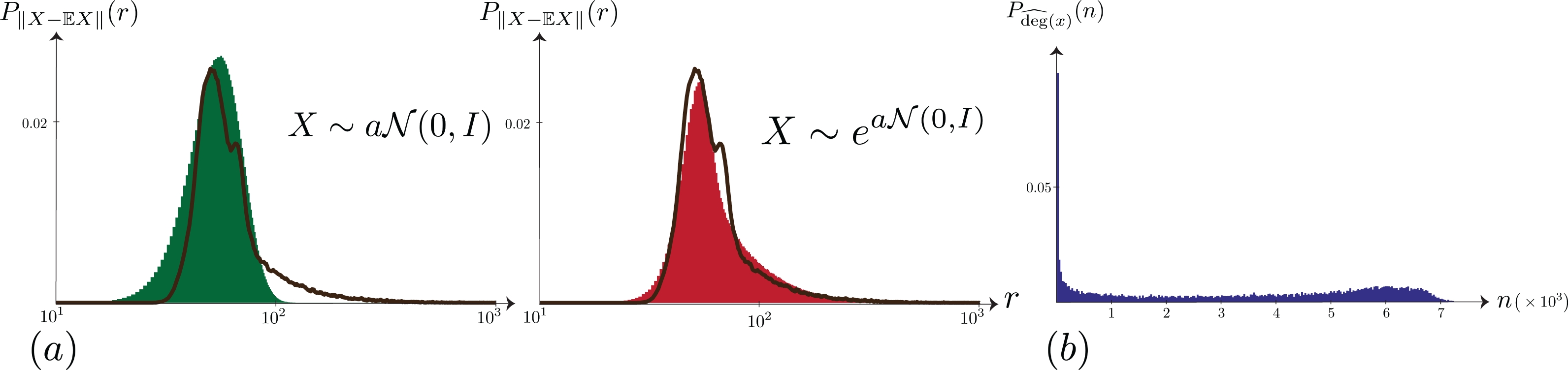}
  \caption{$(a)$ Left: The empirical distribution of the norm of centered cell vectors (brown), and distribution of the norm of Gaussian points (green) for $d = 8$, and $a = 21.5$; $(a)$ Right: The distribution of the norm of centered log normal points (red) for $d = 7$ and $a = 20$. $(b)$: The distribution of \emph{soft} degrees of each node, i.e., $\widehat{\mathrm{deg}}(x_n) = \sum_{m \in [N]} A_{nm}$.}
  \label{fig:prior}
\end{figure}

{\bf $\ell_2$ distances:} Let $d \in \set{2, 10, 10^2, 10^3, 10^4}$, and sub-cliques of size $N = 20$. For a space form of dimension $d$, we iterate over a set of scale parameters $a$, and pick the point distribution $P_{X}$ that produces the closest ordinal spread distribution to the empirical distribution of $\alpha_{20}$. We show that our proposed test detects the Euclidean space of dimension $100$ as the geometry of $\ell_2$ comparisons. However, higher dimensional Euclidean spaces could also be plausible choices; see \Cref{fig:l2}.

 \begin{figure}[H]
\center
  \includegraphics[width=1\linewidth]{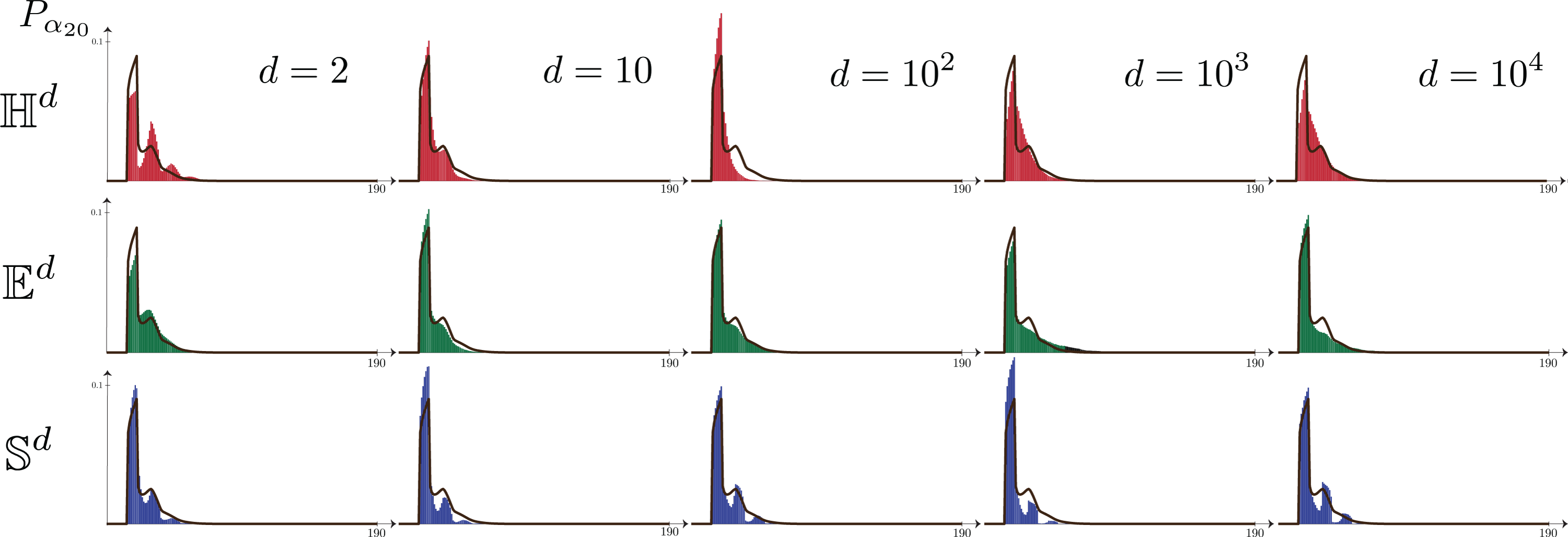}
  \caption{Brown curves represent the empirical PMF of $\alpha_{20}$ derived from $\ell_2$ comparisons. In each row, we generate points from the corresponding space form of dimension $d$. This test reveals a high-dimensional Euclidean space for the geometry of $\ell_2$ comparisons ($d^* = 100$).}
  \label{fig:l2}
\end{figure}

{\bf Angluar distance:} We repeat the experiment for angle comparisons explained earlier. Our proposed test identifies $10000$-dimensional spherical space as the geometry of angle comparisons; see \Cref{fig:angle}.

 \begin{figure}[H]
\center
  \includegraphics[width=1\linewidth]{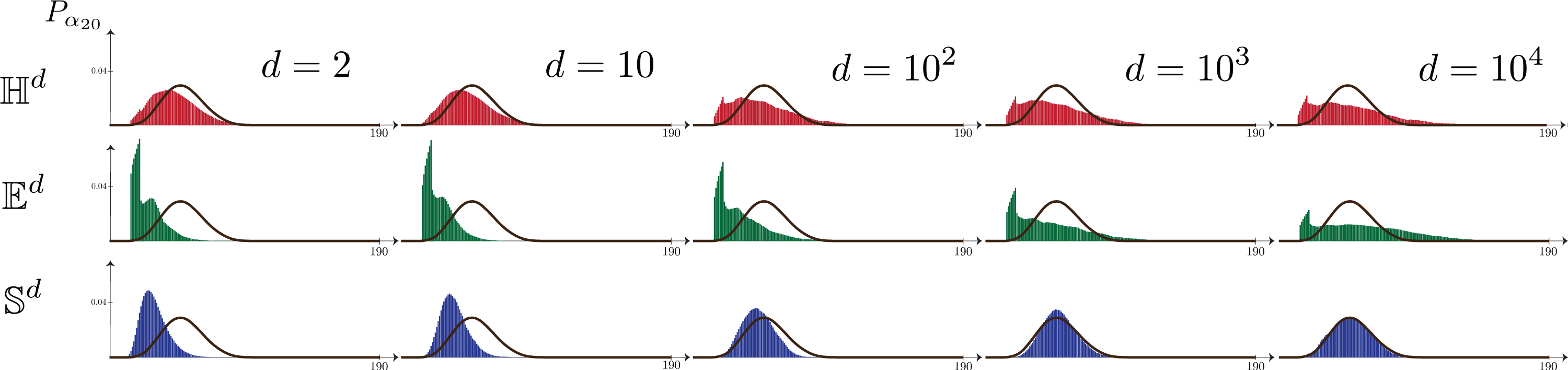}
  \caption{Brown curves represent the empirical PMF of $\alpha_{20}$ from angle comparisons. In each row, we generate points from the corresponding space form of dimension $d$. This test reveals a high-dimensional spherical space for the geometry of angle comparisons ($d^* = 10000$).}
  \label{fig:angle}
\end{figure}

{\bf RFA index comparisons:} This test identifies $1000$-dimensional spherical space as the geometry of RFA index comparisons; see \Cref{fig:RFA}. For embedding ordinal measurements, we pick a random clique of size $200$ and embed it in low-dimensional space forms of different dimensions. Then, we compute the empirical probability of erroneous comparison, i.e., error occurs if $d(x_i,x_j) \geq d(x_k, x_l)$ whereas the points $x_i,x_j$ are more similar to each other compared to the points $x_k, x_l$. We repeat the experiment $200$ times, and report the mean and standard deviations of the probability of error $p_e$. 
 \begin{figure}[b]
\center
  \includegraphics[width=1\linewidth]{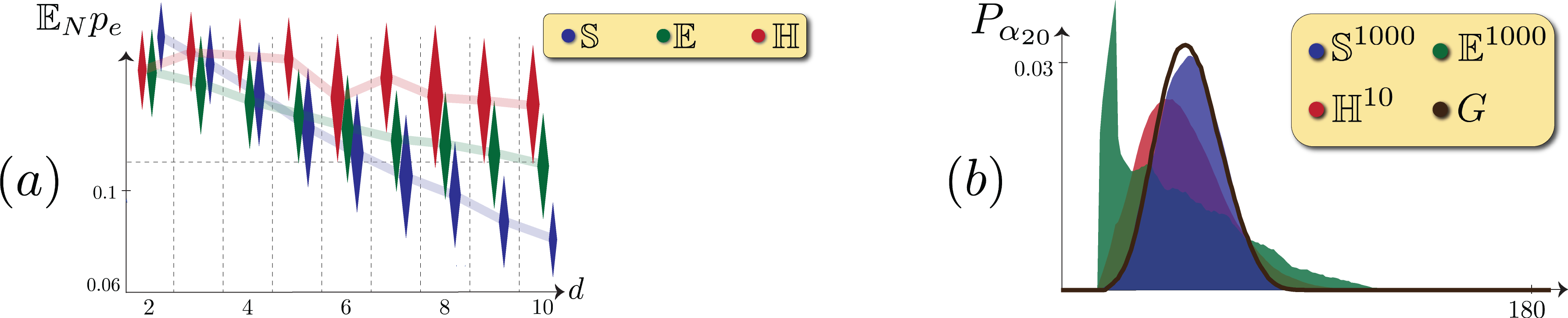}
  \caption{$(a)$ PMFs of $\alpha_{20}$ from the RFA similarity graph ($G$) vs. random points in space forms of optimal dimensions. $(b)$ $\mathbb{E}_N p_e$ for embedded points in $d$-dimensional space forms.}
  \label{fig:RFA}
\end{figure}
An important observation is that higher dimensional of space forms do not necessarily give better matches for the empirical PMF of ordinal spread variables.

\section{EMBEDDING ALGORITHMS}\label{sec:EMBEDDING_ALGORITHMS}
We can use semidefinite programs to solve non-metric embedding problems in hyperbolic and Euclidean spaces \cite{tabaghi2020hyperbolic,agarwal2007generalized}. The main objects in these problems are distance matrices, and the matrix of inner products, e.g., Gramian in Euclidean space and Lorentzian matrix in hyperbolic space. The traditional interior point method to solve semidefinite programs do not scale to large problems. This is especially the case for non-metric embedding problems in which we have ${{N \choose 2} \choose 2} = O(N^4)$ distinct inequality constraints related to pairwise distance comparisons. Therefore, we propose non-metric embedding algorithms based on the method of alternative projections; see \Cref{alg:HyperbolicEmbedding,alg:EuclideanEmbedding,alg:SphericalEmbedding}.

\subsection{Hyperbolic Embedding}
We start with an arbitrary hyperbolic distance matrix (refer to \cite{tabaghi2020hyperbolic}), and a sorted index list. The function $\mathtt{IndexList}(D)$ computes the index list associated with the distance matrix $D$. 

We begin with arranging the elements of $D$ according to the target index list $(i,j)$. In other words, we have 
\[
\mathtt{sort}(D, (i,j)) = (d_{\pi(i_r,j_r)})_{i_r,j_r \in [N]}
\] where $\pi: [N]^2 \rightarrow [N^2]$ is a one-to-one map, such that $\pi(i_r,j_r) = \pi(j_r,i_r)$, $\pi(i_r,i_r) = (i_r,i_r)$, and $\mathtt{IndexList}( \mathtt{sort}(D, (i,j) ) ) = (i,j)$. The resulting symmetric matrix is no longer a valid hyperbolic distance matrix. Therefore, we proceed with finding the best rank-$(d+1)$ Lorentzian matrix--- the matrix of Lorentizan inner products. We compute the corresponding point set, in $\R^{d+1}$, by a simple spectral factorization of the Lorentizan matrix; see \Cref{alg:HyperbolicEmbedding} lines $6-8$ and refer to \cite{tabaghi2020hyperbolic}. Finally, we use a simple method to map each point (columns of $X$) to $\bL^d$, viz.,
\[
P_{\R^{d+1} \rightarrow \bL^d}(x) = \begin{bmatrix}
\sqrt{1+\norm{y}^2} \\ y
\end{bmatrix} \ \mbox{where} \ y = (x_2, \ldots x_{d+1})^\T.
\]
We compute the hyperbolic Gramian, $G = X^\T H X$, where $H = \mathrm{diag}(-1,1, \ldots, 1)\in \R^{(d+1) \times (d+1)}$. This gives us the update for hyperbolic distance matrix $D = \acosh[-G]$; refer to \cite{tabaghi2020hyperbolic}. We repeat this process till convergence. 
\begin{algorithm}[H]
\caption{Non-metric hyperbolic embedding}
\begin{algorithmic}[1]
\Procedure{HyperbolicEmbedding}{$(i,j),d$} 
	\State {\bf input:} Index list $(i,j)$, and embedding dimension $d$.
    \State {\bf initialize:} hyperbolic distance matrix $D$, and an arbitrary index list $(\tilde{i},\tilde{j})$.
    \While{$\norm{ (\tilde{i},\tilde{j}) - \mathtt{IndexList}(D)} > 0$}
    	\State $(\tilde{i},\tilde{j}) \gets \mathtt{IndexList}(D)$. \Comment{The index list related to $D$.}
    	\State $D \gets  \mathtt{sort}(D, (i,j) )$. \Comment{Update $D$ by sort distances according to $(i,j)$.}
    	\State Let $U \Sigma U^\T$ be the eigenvalue decomposition of $G= -\cosh [D]$ such that $\sigma_1 \geq \ldots \geq \sigma_N \in \R$.
    	\State $X = |\Sigma_d|^{1/2} U_d^\T$, where $\Sigma_d = \mathrm{diag} [ (\sigma_1)_{+}, \ldots, (\sigma_d)_{+}, (\sigma_N)_{-}]$ and $U_d$ is the sliced eigenvector matrix.
    	\State $X \gets P_{\R^{d+1} \rightarrow \bL^d}( X)$. \Comment{Map each column of $X \in \R^{(d+1)\times N}$ to $\bL^d$.}
    	\State $G = X^\T H X$. \Comment{Hyperbolic Gramian.}
    	\State $D \gets \acosh [-G]$. \Comment{Hyperbolic distance matrix.}
    \EndWhile
	\State {\bf return} $X$
\EndProcedure
\end{algorithmic}\label{alg:HyperbolicEmbedding}
\end{algorithm}

\subsection{Spherical Embedding}
We propose a similar method for spherical embedding. The matrix of inner products $G$ and the spherical distance matrix $D$ are related via $D = \acos [G]$. For points in $d$-dimensional spherical space, the matrix $G$ is a positive semidefinite matrix of rank $(d+1)$, and with diagonal elements of $1$. The spectral factorization of $G$ gives us the point positions.

At each iteration of \Cref{alg:SphericalEmbedding}, we shuffle the elements of the distance matrix according to the target index list, but the corresponding Gram matrix $G = \cos[ \mathtt{sort}(D, (i,j) ) ]$ is not a valid Gram matrix for points in $\bS^d$. We first find the best rank-$(d+1)$ positive semidefinite matrix via a simple eigenvalue thresholding which gives a set of points in $\R^{d+1}$ --- as opposed to $\bS^d$; see line $9$ of \Cref{alg:SphericalEmbedding}. Therefore, we radially project each point to $\bS^d$, i.e., $P_{\R^{d+1} \rightarrow \bS^d}(x) = \frac{1}{\norm{x}} x$. We repeat this process till a convergence is achieved. 
\begin{algorithm}[H]
\caption{Non-metric spherical Embedding}
\begin{algorithmic}[1]
\Procedure{SphericalEmbedding}{$(i,j),d$}      
    \State {\bf input:} Index list $(i,j)$, and embedding dimension $d$.
    \State {\bf initialize:} Spherical distance matrix $D$, and an arbitrary index list $(\tilde{i},\tilde{j})$.
    \While{$\norm{ (\tilde{i},\tilde{j}) - \mathtt{IndexList}(D)} > 0$}
    	\State $(\tilde{i},\tilde{j}) = \mathtt{IndexList}(D)$. \Comment{The index list related to $D$.}
		\State $D \gets \mathtt{sort}(D, (i,j) )$.  \Comment{Update $D$ by sort distances according to $(i,j)$.}
    	\State Let $U \Sigma U^\T$ be eigenvalue decomposition of $G= \cos [D]$ such that $\sigma_1 \geq \ldots \geq \sigma_N \in \R$.
    	\State Let $\Sigma_d = \mathrm{diag} [ (\sigma_1)_{+}, \ldots, (\sigma_{d+1})_{+}]$, and $U_d$ be corresponding eigenvector matrix.
    	\State $X = \Sigma_d ^{1/2} U_d^\T$.
    	\State $X \gets P_{\R^{d+1} \rightarrow \bS^d}( X)$. \Comment{Map each column of $X \in \R^{(d+1)\times N}$ to $\bS^d$.}
    	\State $G = X^\T X$. \Comment{The Gram matrix.}
    	\State $D \gets \acos [G]$.
    \EndWhile 
    \State {\bf return} $X$
\EndProcedure
\end{algorithmic}\label{alg:SphericalEmbedding}
\end{algorithm}
\subsection{Euclidean Embedding}
Unlike hyperbolic and spherical counterparts, Euclidean distance matrix $D \in \R^{N \times N}$ is the matrix of squared distances between a set of $N$ points $X \in \R^{d \times N}$. This definition lets us to express it as a linear function of the Gram matrix $G = X^\T X$, i.e., $D = \mathcal{K}(G) = -2G + \mathrm{diag}(G) 1^\T + 1 \mathrm{diag}(G)^\T$, where $\mathrm{diag}(G)$ is a vector of diagonal elements of $G$, and $1 \in \R^{N}$ is the vector of all ones. The Gram matrix $G$ is positive semidefinite of rank at most $d$. We can find the centered Gramian from a given distance matrix as $G = -\frac{1}{2} J D J$, where $J = I - \frac{1}{N}11^\T$.  At each iteration of \Cref{alg:EuclideanEmbedding}, we find the best rank-$d$ positive semidefinite matrix via a simple eigenvalue thresholding of $G = -\frac{1}{2} J D J$; see lines $7-8$ of \Cref{alg:EuclideanEmbedding}. The spectral factorization of $G$ gives the point set in $\R^{d}$. We repeat this process until convergence. 
\begin{algorithm}[H]
\caption{Non-metric Euclidean embedding}
\begin{algorithmic}[1]
\Procedure{EuclideanEmbedding}{$(i,j),d$}      
	\State {\bf input:} Index list $(i,j)$, and embedding dimension $d$.
    \State {\bf initialize:} hyperbolic distance matrix $D$, and an arbitrary index list $(\tilde{i},\tilde{j})$.
    \While{$\norm{ (\tilde{i},\tilde{j}) - \mathtt{IndexList}(D)} > 0$}
    \State $(\tilde{i},\tilde{j}) \gets \mathtt{IndexList}(D)$. \Comment{The index list related to $D$.}
    \State $D = \mathtt{sort}(D, (i,j) )$.
    \State Let $U \Sigma U^\T$ be the eigenvalue decomposition of $G= -\frac{1}{2} J D J$ such that $\sigma_1 \geq \ldots \geq \sigma_N \in \R$.
    \State $G = U_d\Sigma_d  U_d^\T$, where $\Sigma_d = \mathrm{diag} [ (\sigma_1)_{+}, \ldots, (\sigma_d)_{+}]$ and $U_d$ is the sliced eigenvector matrix.
    \State $D \gets \mathcal{K}(G)$. \Comment{Euclidean distance matrix.}
    \EndWhile
    \State {\bf return} $X = \Sigma_d^{1/2} U_d^\T$.
\EndProcedure
\end{algorithmic}\label{alg:EuclideanEmbedding}
\end{algorithm}
\vfill

\end{document}